\documentclass[oneeqnum, onethmnum, onefignum, onetabnum]{siamltex}

\pdfoutput=1

\usepackage{graphicx}  % Written by David Carlisle and Sebastian Rahtz
\usepackage{amsmath}   % From the American Mathematical Society
  \usepackage{amssymb}
 \usepackage{mathrsfs}
 \usepackage{mathabx}
\usepackage{mathdots}
\usepackage{cite}
\usepackage[ruled,lined,boxed]{algorithm2e}

\newcommand{\bfy}{{\bf y}}
\newcommand{\bfx}{{\bf x}}
\newcommand{\bfz}{{\bf z}}

\newcommand{\bfv}{{\bf v}}
\newcommand{\bfw}{{\bf w}}
\newcommand{\bfr}{{\bf r}}

\newcommand{\bfs}{{\bf s}}

\newcommand{\reals}{{\mathbb R}}
\newcommand{\wcheck}{\widecheck}
\newcommand{\wtilde}{\widetilde}

\newcommand{\svd}{{\rm SVD}}

\newtheorem{thm}{Theorem}[section]

\newtheorem{lma}[thm]{Lemma}
\newtheorem{dfn}[thm]{Definition}

\author{Aswin C. Sankaranarayanan\thanks{A.~C.~Sankaranarayanan is with the Department of Electrical and Computer Engineering at Carnegie Mellon University, Pittsbugh, PA. Email: saswin@ece.cmu.edu} 
        \and Pavan K. Turaga\thanks{P.~K.~Turaga is with the Arts Media and Engineering Department at the Arizona State University, Tempe, AZ. Email: pturaga@asu.edu}
        \and Rama Chellappa\thanks{R.~Chellappa is with the Department of Electrical and Computer Engineering at the University of Maryland, College Park, MD. Email: rama@cfar.umd.edu.} \and Richard G. Baraniuk\thanks{R.~G.~Baraniuk is with the Department of Electrical and Computer Engineering at the Rice University, Houston, TX. Email: richb@rice.edu. Web: dsp.rice.edu}}

\title{Compressive Acquisition of Linear Dynamical Systems}

\begin{document}

\maketitle

\begin{abstract}
Compressive sensing (CS) enables the acquisition and recovery of sparse signals and images at sampling rates significantly below the classical Nyquist rate.  Despite significant progress in the theory and methods of CS, little headway has been made in compressive video acquisition and recovery.  Video CS is complicated by the ephemeral nature of dynamic events, which makes direct extensions of standard CS imaging architectures and signal models difficult. In this paper, we develop a new framework for video CS for dynamic textured scenes that models the evolution of the scene as a linear dynamical system (LDS). This reduces the video recovery problem to first estimating the model parameters of the LDS from compressive measurements and then reconstructing the image frames.
We exploit the low-dimensional dynamic parameters (the state sequence) and high-dimensional static parameters (the observation matrix) of the LDS to devise a novel compressive measurement strategy that measures only the time-varying parameters at each instant and accumulates measurements over time to estimate the time-invariant parameters. This enables us to  lower the compressive measurement rate considerably.
We validate our approach and demonstrate its effectiveness with a range of experiments involving  video recovery and scene classification.

\end{abstract}


%\begin{AMS}
%\end{AMS}

\pagestyle{myheadings}
\thispagestyle{plain}
\markboth{ACS, PKT, RC and RGB }{CS-LDS}

\section{Introduction} \label{sec:intro}

The Shannon-Nyquist theorem dictates that to sense features at a particular frequency, we  
must sample uniformly at twice that rate. 
For generic imaging applications, this sampling rate might be too high; in modern digital cameras, invariably, the sensed imaged is compressed immediately without much loss in quality.
For other applications, such as high speed imaging and sensing in the non-visual spectrum, camera/sensor designs based on the Shannon-Nyquist theorem lead to impractical and costly designs.
Part of the reason for this is that the Shannon-Nyquist sampling theory does not exploit any structure in the sensed signal beyond that of band-limitedness.
Signals with redundant structures can potentially be sensed more parsimoniously.
This is the key idea underlying the new field of {\em compressive sensing} (CS) \cite{candes2006robust}.
When the signal of interest exhibits a sparse representation, CS enables sensing at measurement rates below the Nyquist rate. 
%CS exploits the property that the signal of interest is often sparse in some transform basis in order to recover it from a small number of linear, random, multiplexed measurements. 
Indeed, signal recovery is possible from a number of measurements that is proportional to the sparsity level of the signal, as opposed to its bandwidth.
%CS enables compression at sensing which is useful in reducing sensing costs dramatically. 
%This also manifests itself in many ways including alleviating the data deluge associated with the processing and storage of videos.

In this paper, we consider the problem of sensing \emph{videos} compressively.
%Video CS is a compelling problem for two reasons. 
We are interested in this problem motivated by the success of video compression algorithms, which  indicates that videos are highly redundant.
Bridging the gap between compression and sensing can lead to compelling camera designs that  significantly reduce the amount of data sensed and enable designs for application domains where sensing is inherently costly.
% Second, video cameras are increasingly common in
%In most of these, the sensed data is stored 

Video CS is challenging for two main reasons:
\begin{itemize}
\item {\bf Ephemeral nature of videos: } The scene changes during the measurement process; moreover, we cannot obtain additional measurements of an event after it has occurred.
\item {\bf High-dimensional signals: } Videos are significantly higher-dimensional than images.
This makes the recovery process computationally intensive. 
\end{itemize}
One way to address these challenges is to narrow our scope to certain parametric models that are suitable for a broad class of videos; this morphs the video recovery problem to one of parameter estimation and provides a scaffold to address the challenges listed above.

In this paper, we develop a  CS framework for videos modeled as linear dynamical systems (LDSs), which is motivated, in part, by the extensive use of such models in characterizing dynamic textures \cite{Chan2005,doretto2003dynamic,saisanDWS01}, activity modeling, and video clustering \cite{turaga2009unsupervised}.
%These aspects set the proposed solution apart from the existing approaches to compressive video sensing.
%Much of this limitation comes from a lack of a strong signal model describing the video. It is important to model a video as an object greater than a sequence of frames with some smoothness between adjacent frames. A key drawback of such an interpretation is that the complexity of such a model depends on the (temporal) sampling rate of the video. A reconstruction at twice the frame rate would necessarily require twice the number of measurements. However, for a large class of dynamic scenes, the inherent information content/complexity depends more on the duration of the video (say in seconds) and less on its sampling rate. It is important to use video models and priors that respect this property.
Parameteric models, like LDSs, offer lower dimensional representations for otherwise high-dimensional videos. This significantly reduces the number of free parameters that need to be estimated and,
as a consequence, reduces the amount of data that needs to be sensed.
In the context of  video sensing, LDSs offer interesting tradeoffs by characterizing the video signal using a mix of dynamic/time-varying parameters and static/time-invariant parameters.
Further, the generative nature of LDSs provides a prior for the evolution of the video in both forward and reverse time. 
To a large extent, this property helps us circumvent the challenges presented by the ephemeral nature of videos.

%In this paper, we propose a CS acquisition and recovery framework for videos modeled as LDSs; for LDSs, recovering the video is the same as estimating its parameters. 
%However, these parameters appear in bilinear terms --- which typically require sophisticated non-convex techniques for recovery.
The paper makes the following contributions. 
We propose a framework called {\em CS-LDS} for video acquisition using an LDS model coupled with sparse priors for the parameters of the LDS model.
The core of the framework is a {\em two-step measurement strategy} that enables the recovery of the LDS parameters from compressive measurements by solving a sequence of linear and convex problems.
%We derive sufficient conditions for the recov
% To circumvent this problem, we propose a novel measurement process that enables recovery of the LDS parameters by solving a sequence of linear and convex problems.
% We solve for the parameters of the LDS using an efficient recovery algorithm that exploits structured sparsity patterns in the observation matrix. 
%We highlight examples of scene and activity classification performed directly on the recovered system parameters, foregoing the need for actually reconstructing the frames of the sensed video. 
We demonstrate that CS-LDS is capable of sensing videos with far fewer measurements than the Nyquist rate.
Finally, the LDS parameters form an important class of features for activity recognition and scene analysis, thereby making our camera designs purposive \cite{nayar2006programmable} as well.

%{\flushleft{\bf Specific contributions:}} We make the following contributions:
%\begin{itemize}
%\item A framework called {\em CS-LDS} for video acquisition using a LDS model coupled with sparse priors for the parameters of the LDS model.
%\item A two-step measurement strategy that enables the recovery of LDS parameters directly from compressive measurements.
%\item An efficient recovery algorithm that exploits structured sparsity patterns in the observation matrix.
%\item Stable recovery at a measurement rate inversely proportional to the duration of the acquired video.
%\end{itemize}
%We also highlight examples of scene classification performed directly on the recovered system parameters, foregoing the need for actually reconstructing the frames of the sensed video. Finally, we show video reconstruction and synthesis results over a wide range of acquisition conditions.

\section{Background} \label{sec:prior}
%We will use the following notation.
%At time instant $t$, we use $\bfy_t$ to denote the image or video frame, $\bfx_t$ to denote the LDS state variables  and $\bfz_t$ to denote compressive measurements.
%We use $\Phi_t$ to denote the measurement matrix, i.e $\bfz_t = \Phi_t \bfy_t$.
%We use the notation $a:b$ to denote the set $\{a, a+1, a+2, \cdots, a+b\}$.
%We use subscripts to denote sequences
%such as $\bfx_{1:T} = \{\bfx_1, \bfx_2, \ldots, \bfx_T\}$. We use  $[ \cdot ]$ to denote matrices,
%such as $[\bfy_{1:T}]$ is the $N \times T$ matrix formed by $\bfy_{1:T}$ such that the $k$-th column is $\bfy_k$.

%We discuss related work and highlight ideas that are relevant to those presented in this paper.

\subsection{Compressive sensing} \label{sec:cs}
CS  deals with
the recovery of a signal $\bfy \in \reals^N$ from undersampled linear measurements
of the form $\bfz = \Phi \bfy + {\bf e}$, where $\Phi \in \mathbb{R}^{M \times N}$ is the measurement matrix, $M < N,$ and ${\bf e}$ is the measurement noise \cite{candes2006robust,donoho2006compressed}.
Estimating $\bfy$ from the measurements $\bfz$ is ill-conditioned, since the linear system formed by $\bfz = \Phi \bfy$ is under-determined.
CS works under the assumption that the signal $\bfy$ is sparse in a basis $\Psi$; that is, the signal $\bfs$, defined as $\bfy = \Psi \bfs$, has at most $K$ non-zero components.
Exploiting the sparsity of $\bfs$, the signal $\bfy$ can be recovered exactly from $M = O(K\log(N/K))$ measurements provided the matrix $\Phi\Psi$ satisfies
the so-called {\em restricted isometry property} (RIP) \cite{baraniuk2008simple}.
In particular, when $\Psi$ is an orthonormal basis
and the entries of the matrix $\Phi$ are  i.i.d.\ samples from a sub-Gaussian distribution,
the product $\Phi\Psi$ satisfies the RIP.
Further, the signal $\bfy$ can be recovered from $\bfz$ by solving a convex problem of the form
\begin{equation}
\mbox{ } \min \| \bfs \|_1 \textrm{ subject to  } \| \bfz - \Phi \Psi \bfs \|_2 \le \epsilon,
\label{eqn:l1prob}
\end{equation}
where $\epsilon$ is an upper bound on the measurement noise ${\bf e}$. It can be shown that the solution to (\ref{eqn:l1prob}) is with
high probability the $K$-sparse solution that we seek. The theoretical guarantees of CS have been extended to {\em compressible} signals,
where the sorted coefficients of $\bfs$ decay rapidly according to a power-law \cite{haupt2006signal}.

There exist a wide range of algorithms to solve (\ref{eqn:l1prob}) under various
approximations or reformulations \cite{candes2006robust,berg2008probing}.
%It is also possible to solve (\ref{eqn:l1prob}) efficiently using
Greedy techniques
such as Orthogonal Matching Pursuit \cite{pati1993orthogonal} and
CoSAMP \cite{needell2009cosamp} solve the sparse approximation problem efficiently with strong convergence properties and low computational complexity.
It is also simple to impose structural constraints such as block sparsity into CoSAMP,
giving variants such as model-based CoSAMP \cite{baraniuk2008model}.
% we use an instance of model-based CoSAMP in Section \ref{sec:obs}.

\subsection{Video compressive sensing} \label{sec:videocs}
In this paper, we model a video as a sequence of time-indexed images. Specifically, if $\bfy_t$ is the image of a scene at time $t$, then $\bfy_{1:T} = \{ \bfy_1, \ldots, \bfy_T \}$ is the video of the scene from time $1$ to $T$. Further, we also refer to $\bfy_t$ as the ``video frame'' at time $t$.

In video CS, the goal is to sense a time-varying scene using
compressive measurements of the form $\bfz_t = \Phi_t \bfy_t$, where $\bfz_t, \Phi_t$ and $\bfy_t$ are the compressive measurements, the measurement matrix and the video frame at time $t$, respectively.
Given the sequence of compressive measurements $\bfz_{1:T} = \{ \bfz_1, \bfz_2, \ldots, \bfz_T\}$, our goal is to recover the video  $\bfy_{1:T} = \{ \bfy_1, \bfy_2, \ldots, \bfy_T\}$.
%The key challenge here is that the scene changes as we sense it. 
There are currently two fundamentally different imaging architectures for video CS: the single pixel camera (SPC) and the programmable pixel camera. 
%The first approach assumes a single pixel camera (SPC) \cite{duarte2008single}
%that provides full spatial multiplexing of the image at each time instant.
%The second approach assumes a full frame sensor that provides temporal multiplexing over a duration of time, but with limited spatial multiplexing. We discuss these two approaches next.
The SPC \cite{duarte2008single} uses a single or a small number of sensing elements. Typically, a photo-detector is used to obtain a single measurement at each time instant of the form $\bfz_t = \phi_t^T \bfy_t$, where $\phi_t$ is a pseudo-random
vector of $0$s and $1$s.
Typically, under an assumption of a slowly varying scene, consecutive measurements from the 
SPC are grouped as measurements of the same video frame.
This assumption works only when the scene motion is small or when the  number
of measurements associated with a frame is small.
The SPC provides complete freedom in the spatial multiplexing of pixels; however, there is
no temporal multiplexing.
In contrast, programmable pixel cameras \cite{veeraraghavan2011coded, reddy2011p2c2, hitomi2011video} use a 
full frame sensor array; during each exposure of the sensor array, the shutter at each pixel is 
temporally modulated. This enables extensive temporal multiplexing but a limited amount of 
spatial multiplexing.
%These two imaging architectures for video CS methods offer interesting tradeoffs.
%Further, SPCs provide greater flexibility in the choice of measurement vectors $\Phi_t$;
%measurement vectors whose entries take  $0/1$  or $\pm 1$ values are implementable on an SPC.
%SPCs enable a high degree of spatial multiplexing, while  programmable pixel cameras provide a high degree of temporal multiplexing.
%The SPC takes only a single measurement per time instant, which necessitates grouping measurements under the assumption of a slowly varying scene.
%In contrast, full frame sensor methods 
% multiple compressive measurements at each time instant.
A key advantage of SPC-based designs is that they can operate efficiently at wavelengths
(such as the far infrared) that require exotic detectors; in such cases, building a full frame sensor can be prohibitively expensive.

To date, recovery algorithms for the SPC have used various signal models to reconstruct the sensed scene.
Wakin et al. \cite{wakin2006compressive} use 3D wavelets as the sparsifying basis for recovering 
videos from  compressive measurements. 
Park and Wakin \cite{park2009multiscale}
use a coarse-to-fine estimation framework wherein the video, reconstructed at a coarse scale, is 
used to estimate motion vectors that are subsequently used to design dictionaries for reconstruction at a finer scale.
Vaswani \cite{vaswani2008kalman} and Vaswani and Lu \cite{vaswani2009modified} use a
sequential framework that exploits the similarity of support of the signal between adjacent frames of a video.
Under this model, a frame of video is reconstructed using a linear inversion over the support at the previous time instant
and a small-scale CS recovery over the residue to detect components beyond the known support.
Cevher et al. \cite{cevher2008compressive} provide a CS framework for directly sensing innovations over a static scene thereby enabling background subtraction from compressive measurements.

\begin{figure}[!ttt]
\centering
\includegraphics[width=\textwidth]{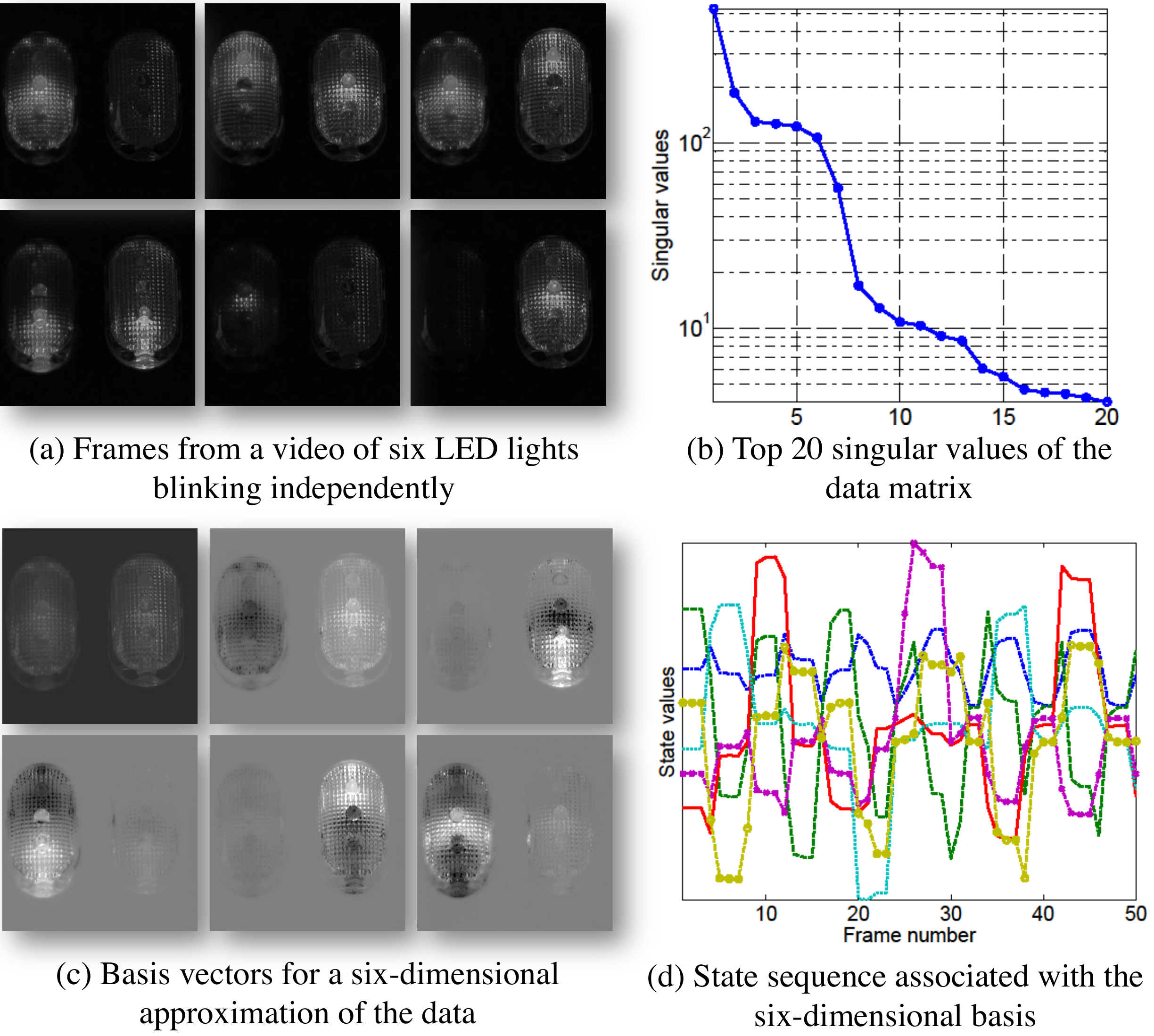}
\caption{An example of an LDS and the models that define it. (a) A few frames of a video of six LEDs flashing independently. (b) Top 20 singular values of the data matrix $[ \bfy_{1:T} ]$ --- formed by stacking  frames of the video as its columns. Note how the singular values, outside the top six, decay rapidly. This is a consequence of the linear nature of light suggests that the frames of the video lie on a six-dimensional subspace. In practice, deviations from linearity   due to saturation lead to small deviations from the six-dimensional subspace as noted from the decaying singular values. 
(c) Basis vectors associated with a six-dimensional approximation of the data. Blacker pixels denote non-negative entries while whiter pixels denote positive entries. Together, they define a six-dimensional subspace that defines the observation model of the LDS.
(d) State sequence associated with the six-dimensional approximation. The smooth variation of the state values indicate predictability over small time durations --- one of the key hallmarks of an LDS. These smooth transitions are captured by the state transition model.}
\label{fig:led_example}
\end{figure}

\subsection{Linear dynamical system model for video sequences} \label{sec:priorlds}
Linear dynamical systems (LDSs) represent an important class of parametric models for time-series data.
A wide variety of spatio-temporal signals have often been modeled as realizations of LDSs.
These include dynamic textures \cite{doretto2003dynamic}, traffic scenes \cite{Chan2005}, video inpainting \cite{ding2007rank}, multi-camera tracking \cite{ayazoglu2011dynamic} and human activities \cite{turaga2009unsupervised}. The interested reader is  referred to \cite{sznaier2012compressive} for a survey of the use of LDSs as a concise representation for a wide range of computer vision problems.
%Let $\bfy_t \in \reals^N$ be the video frame at time $t$ and
%let $\bfy_{1:T}$ be a  sequence of video frames; here we follow the notation introduced earlier in section \ref{sec:videocs}. 

Intuitively, a LDS for a video comprises of two models. First, an {\em observation model} that suggests that frames of the video lie close to a $d$-dimensional subspace; the frame of the video at time $t$ can be represented as $\bfy_t \approx C \bfx_t,$ where $C$ is a basis for the subspace and $\bfx_t$ are the subspace coefficients or the state vector at time $t$. Second, the trajectory that the video charts out in this $d$-dimensional subspace varies smoothly, is predictable and modeled by a linear evolution of the form $\bfx_{t+1} \approx A \bfx_{t}$.
Figure \ref{fig:led_example} provides an  example of an LDS.

We now formally define the LDS for a video. The model equations are given by 
\begin{align}
\bfy_t    =&   C \bfx_t  +   \bfw_t,  &\bfw_t    \sim    N(0,R) \\
\bfx_{t+1} =&   A \bfx_t  +   \bfv_t, &   \bfv_t    \sim    N(0,Q) \label{eqn:markov_process}
\end{align}
where $\bfx_t \in {\mathbb R}^d$  is the  state vector at time $t$,  $d$ is the dimension of the state space,
$A \in {\mathbb R}^{d\times d}$ is the state transition matrix, $C \in {\mathbb R}^{N \times d}$ is the observation matrix,  $\bfy_t \in \mathbb{R}^N$ represents the observed measurements, where for the videos of interest in this paper, $d \ll N$. $\bfw_t$ and $\bfv_t$ are noise components modeled as Gaussian with $0$ mean vector and covariance matrices given by $R \in \mathbb{R}^{N \times N}$ and $Q \in \mathbb{R}^{d \times d}$, respectively.  The Gaussian assumption for the process noise is not necessarily an optimal one, but is made for the sake of simplifying the model estimation algorithm. It is known to work well for representing a large class of dynamic textures \cite{doretto2003dynamic}. 

An LDS is  parameterized by the matrix pair $(C, A)$. Note that the choice of $C$ and the state sequence $\bfx_{1:T}$ is unique only up to a $d \times d$ linear transformation given the inherent ambiguities in the notion of a state space.
In particular, given \emph{any} invertible $d \times d$ matrix $L$, the LDS defined by $(C, A)$ with the state sequence $ \bfx_{1:T} $ is equivalent to the LDS defined by $(CL, L^{-1}A L )$ with the state sequence  $L^{-1} \bfx_{1:T} = \{ L^{-1} \bfx_1, L^{-1} \bfx_2, \ldots, L^{-1}\bfx_T \}$.
This lack of uniqueness has implications that we will touch upon later in Section \ref{sec:obs}.

%In particular, $C$ is normalized such that  $C^T C = {\mathbb I}_d$.

Given a video sequence, the most common approach to fitting an LDS model is to first estimate a lower-dimensional embedding of the observations via principal component analysis (PCA) and then learn the temporal dynamics captured in $\bfx_t$, and equivalently $A$. The most popular model estimation algorithms are N4SID \cite{van1994n4sid}, PCA-ID \cite{soatto2001dynamic}, and expectation-maximization (EM) \cite{Chan2005}. N4SID is a subspace identification algorithm that provides an asymptotically optimal solution for the model parameters. However, for large problems the computational requirements make this method prohibitive.
PCA-ID \cite{soatto2001dynamic} is a sub-optimal solution to the learning problem. It makes the assumption that estimation of the observation matrix $C$ and the state transition matrix $A$ can be separable, which makes it possible to estimate the parameters of the model very efficiently via PCA. Under this assumption, one first estimates the observation matrix $C$, (space-filter) and then uses the result to estimate the state state transition matrix $A$ (time-filter) \cite{doretto2003dynamic}. This learning problem can also be posed as a maximum likelihood estimation of the model parameters that maximize the likelihood of the observations, which can be solved by the EM algorithm \cite{Chan2005}.
%For computational simplicity, we have chosen the PCA based solution of \cite{soatto2001dynamic} in this paper.

%\section{Formulation} \label{sec:problem}
%\input{formulation.tex}
\section{CS-LDS Architecture} \label{sec:cslds}
We provide a high level overview of our proposed framework for video CS; the goal here is  to build a CS framework, implementable on the SPC, for videos that are modeled as LDSs.
We flesh out the details in Sections \ref{sec:coeff} and \ref{sec:obs}.
This amounts to
estimating the LDS parameters from compressive measurements, i.e, we seek to recover the model parameters
 $C$ and $\bfx_{1:T}$ given compressive measurements of the form $\bfz_t = \Phi_t \bfy_t = \Phi_t C \bfx_t$. We recall that $C$ is the time-invariant observation matrix of the LDS, and  $\bfy_t$ and $\bfx_t$ 
 are the video frame and the state at time $t$, respectively.
The compressive measurements $\bfz_{1:T}$ are hence expressed as bilinear terms in the unknown parameters $C$ and $\bfx_{1:t}$. % --- thereby making traditional CS recovery algorithms inapplicable. 
Handling bilinear unknowns typically requires non-convex optimization techniques thereby invalidating conventional CS recovery algorithms.
To avoid this, we propose a two-step sensing method that is specifically designed to address the
bilinearity; we refer to this sensing method and its associated recovery algorithm as the \emph{CS-LDS} framework \cite{sankaranarayanan2010compressive} .

{\flushleft {\bf Measurement model: }} We summarize the CS-LDS measurement model as follows.
At  time $t$, we take two sets of measurements:
\begin{equation}
\bfz_t = \left(
\begin{array}{c} \wcheck{\bfz}_t \\ \wtilde{\bfz}_t \end{array} \right) =
\left[ \begin{array}{c} \wcheck{\Phi} \\ \wtilde{\Phi}_t \end{array} \right] \bfy_t =
\Phi_t \bfy_t,
\label{eqn:meas01}
\end{equation}
where $\wcheck{\bfz}_t \in {\mathbb R}^{\wcheck{M}}$ and $\wtilde{\bfz}_t \in {\mathbb R}^{{\wtilde{M}}}$ such that the total number of measurements at each frame is $M = \wcheck{M}+\wtilde{M}$.\footnote{The SPC obtains only one measurement at each time instant. Multiple measurements for a video frame are obtained by grouping  consecutive measurements from the SPC. When $M$ is small, compared to the sampling rate of the SPC, this is an acceptable approximation especially for slowly varying scenes.}
The measurement matrix in (\ref{eqn:meas01}) is composed of two distinct components: the \emph{time-invariant} part $\wcheck{\Phi}$ and the \emph{time-varying} part $\wtilde{\Phi}_t$.
We denote by $\wcheck{\bfz}_t$ the {\em common} measurements and by $\wtilde{\bfz}_t$ the {\em innovation} measurements.

We solve for the LDS parameters in two steps. First, we obtain an estimate of the state sequence using only the common measurements $\wcheck{\bfz}_{1:T}$. Second, we use this state sequence estimate to recover the observation matrix $C$ using the innovation measurements.

{\flushleft {\bf State sequence estimation: }} We recover the state sequence $\bfx_{1:T}$ using only the common measurements
 $\wcheck{\bfz}_{1:T}$. The key idea is that when $\bfy_{1:T}$ form the observations of an LDS with system matrices $(C, A)$,  the measurements $\wcheck{\bfz}_{1:T}$ form the observations of an LDS with system matrices $(\wcheck{\Phi}C, A)$.
Estimation of the state sequence now can be mapped to a simple exercise in system identification.
In particular, an estimate of the state sequence can be obtained by the singular value decomposition ($\svd$) of the block-Hankel matrix
\begin{equation}
\textrm{Hank}(\wcheck{\bfz}_{1:T}, d)  = \left[
\begin{array}{ccccc}
\wcheck{\bfz}_1 & \wcheck{\bfz}_2 &  \cdots & \cdots & \wcheck{\bfz}_{T-d+1} \\
\wcheck{\bfz}_2 &   \iddots   &  \iddots &  & \wcheck{\bfz}_{T-d+2} \\
\vdots &    \iddots &  \iddots &  \iddots  & \vdots \\
\vdots 		   &                                  & \iddots &  \iddots &      \vdots \\
%\vdots 		   &                       &          & \iddots &   \wcheck{\bfz}_{T-1}           \\
\wcheck{\bfz}_{d} & \cdots & \cdots & \wcheck{\bfz}_{T-1}  & \wcheck{\bfz}_T \\
\end{array}
\right].
\label{eqn:hankelmat}
\end{equation}
Given the $\svd(\textrm{Hank}(\wcheck{\bfz}_{1:T}, d) ) = U_H S_H V_H^T$,  the state sequence estimate is given by
%we obtain the following estimate of the state sequence 
$$[ \widehat{\bfx}_{1:T} ] =  S_H V_H^T.$$
In Section \ref{sec:coeff}, we leverage results from system identification to analyze the properties of this particular estimate as well as characterize the number of measurements $\wcheck{M}$ required.
 
{\flushleft {\bf Observation matrix estimation: }} Given an estimate of the state sequence, $\widehat{\bfx}_{1:T}$, the relationship between the observation matrix $C$ and the innovation measurements  is linear, i.e., $\wtilde{\bfz}_t = \wtilde{\Phi}_t C \widehat{\bfx}_t$.
In addition, $C$ is \emph{time-invariant}. Hence, we can accumulate  innovation measurements over a duration of time to stably reconstruct $C$.
This significantly reduces the number of innovation measurements $\wtilde{M}$ required at \emph{each} frame. This is especially important in the context of sensing videos, since the scene changes as we acquire measurements. Hence, requiring fewer measurements for each reconstructed frame of the video implies less error due to motion blur.

Using the estimates of the state sequence $\widehat{\bfx}_{1:T}$,
we can recover $C$ by solving the following convex problem:
\begin{equation}
\min \sum_{i=1}^d \| \Psi^T {\bf c}_i \|_1  \;\;\; \textrm{s.t.} \;\;\; \forall t, \, \| \bfz_t - \Phi_t  C \widehat{\bfx}_t \|_2 \le \epsilon,
\label{eqn:convex1}
\end{equation}
where ${\bf c}_i$ denotes the $i$-th column of $C$ and $\Psi$ is a sparsifying basis for the 
columns of $C$.
Note that, in (\ref{eqn:convex1}), we use {\em all} of the compressive measurements $\bfz_t$ obtained for each frame of the video --- that is, we use both the common and innovation measurements since the common measurement, much like the innovation measurements, are linear measurements of the frames.
Further, as we show later in Section \ref{sec:strucSpars}, ambiguities in the 
estimation of the state sequence induce a structured sparsity pattern in the support of $C$. 
The convex program (\ref{eqn:convex1}) can be modified to incorporate such constraints. In addition to this, in Section \ref{sec:obs}, we also propose a greedy alternative for solving a variant of the convex program.

\begin{figure*}[!ttt]
\centering
\includegraphics[width=\textwidth]{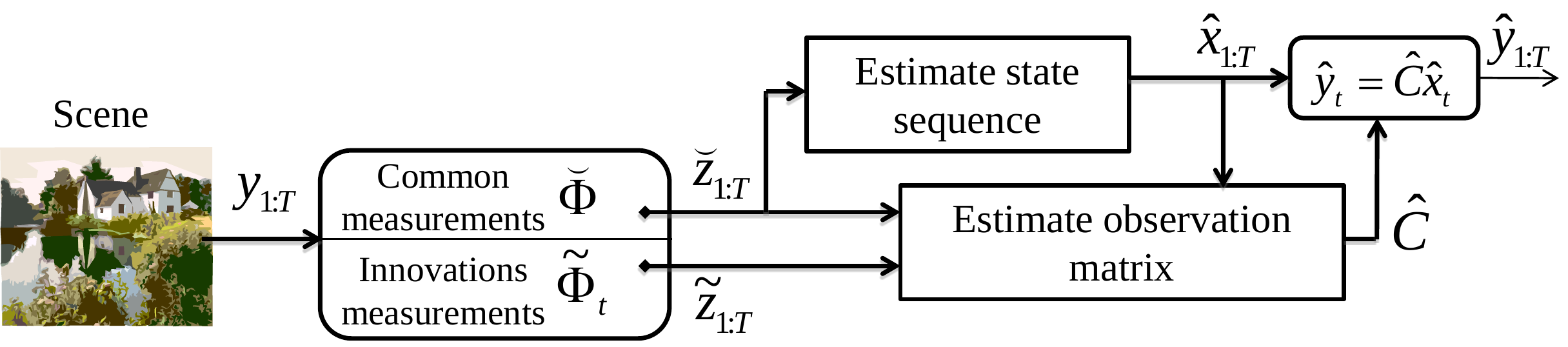}
\caption{{Block diagram of the CS-LDS framework.}}
\label{fig:outline}
\end{figure*}
 
To summarize, the two-step measurement process described in (\ref{eqn:meas01}) enables a two-step recovery (see Figure \ref{fig:outline}). First, we obtain an estimate of the state sequence using SVD on just the common measurements. Second, we use the state sequence estimate for recovering the observation matrix using a convex program. The details of these two steps are discussed in the next two sections.

%Together, these two steps provide a video CS algorithm that uses an LDS prior on the video and its evolution and proposes a measurement scheme and an associated reconstruction algorithm for estimating the LDS parameters directly from compressive measurements. 
%In this regard, this paper extends prior work \cite{sankaranarayanan2010compressive} by borrowing richer techniques from system identification and matrix completion literature for better state space recovery. We establish results on observability of the state sequence given compressive measurements; in this, we modify and extend Wakin et al. \cite{wakin2010observability} to our particular problem. Finally, we demonstrate results on real data as well as hyper-spectral imagery.

\section{Estimating the state sequence} \label{sec:coeff}
In this section, we discuss methods to estimate the state sequence $\bfx_{1:T}$ from the compressive measurements $\wcheck{\bfz}_{1:T}$.
%Given compressive measurements $\wcheck{\bfz}_t = \wcheck{\Phi} \bfy_t  = \wcheck{\Phi} C \bfx_t$, our goal is
%to develop algorithms that estimate the state sequence $\bfx_{1:T}$. In addition to this,
In particular, we seek to establish
sufficient conditions under which the state sequence can be estimated reliably.

\subsection{Observability of the state sequence}\label{sec:observe}
Consider the compressive measurements given by 
\begin{equation}
\wcheck{\bfz}_t = \wcheck{\Phi} y_t + \omega_t,
\label{eqn:common}
\end{equation}
where $\wcheck{\bfz}_t \in \reals^{\wcheck{M}}$ are the compressive measurements at time $t$, $\wcheck{\Phi} \in \reals^{\wcheck{M} \times N}$ is the corresponding measurement matrix,
and $\omega_t \in \reals^{\wcheck{M}}$ is the measurement noise.
Note that $\wcheck{\Phi}$ is time-invariant; hence, (\ref{eqn:common}) is a part of the 
measurement model described in (\ref{eqn:meas01}) relating to the common measurements.
A key observation is that, when $\bfy_{1:T}$ form the observations of an LDS defined by $(C, A)$,  the compressive measurement sequence $\wcheck{\bfz}_{1:T}$ forms an LDS as well; that is,
\begin{eqnarray*}
\wcheck{\bfz}_t &=&  \wcheck{\Phi} y_t + \omega  =  \wcheck{\Phi} C \bfx_t + \omega'_t,\\
\bfx_t 		  &=& A \bfx_{t-1} + w_t.
\end{eqnarray*}		
The LDS  associated with  $\wcheck{\bfz}_{1:T}$ is parameterized by the system matrices $( \wcheck{\Phi} C, A )$.
Estimating the state sequence from the observations of an LDS is possible only when the LDS is \emph{observable} \cite{brockett1970finite}. Thus, it is important to consider the question of observability of the LDS parameterized by  $( \wcheck{\Phi} C, A )$.\footnote{Observability of LDSs in the context of CS has been studied earlier by Wakin et al.\ \cite{wakin2010observability}, who consider the scenario when the observation matrix $C$ is randomly generated and the state vector $\bfx_0$ at $t=0$ is sparse. In contrast, the analysis we present is for a non-sparse state vector.}

\begin{dfn}[Observability of an LDS \cite{brockett1970finite}]
An LDS is \emph{observable} if, for any possible state sequence, the current state can be estimated from a finite number of observations.
\end{dfn}

\begin{lma}[Test for observability of an LDS \cite{brockett1970finite}]
An LDS defined by the system matrices $(C, A)$ and of state space dimension $d$ is observable if and only if the observability matrix
\begin{equation}
(O(C, A))^T = \left[  C^T \, (CA)^T \, (CA^2)^T \, \cdots \, (CA^{d-1})^T  \right]^T
\label{eqn:obsMat}
\end{equation}
is full rank.
\end{lma}

%\begin{lma}
%The LDS defined by $(\wcheck{\Phi}C, A)$ is observable only if the LDS defined by $(C, A)$ is observable.
%\end{lma}
%\begin{proof}
%The observability matrix $O(\wcheck{\Phi} C , A)$ can be written as
%$$O(\wcheck{\Phi} C , A) = \left[ \begin{array}{c} \wcheck{\Phi} C \\\wcheck{\Phi} C A \\ \vdots \\ \wcheck{\Phi} C A^{d-1} \end{array} \right] = 
%\left[
%\begin{array}{cccc} \wcheck{\Phi}  &  &  &  0 \\  & \wcheck{\Phi} &  &  \\ 
% &  & \ddots & \\
%0 &  &  & \wcheck{\Phi} \end{array} \right] O(C, A).$$
%Hence, $\rank( O(\wcheck{\Phi} C, A)) \le \rank(O(C, A))$. 
%If the LDS defined by $(C, A)$ is not observable, then $O(C, A)$ is not full rank; as a consequence, $O(\wcheck{\Phi} C, A)$ is not full rank and, hence, the LDS defined by $(\wcheck{\Phi} C, A)$ is not observable.
%\end{proof}

A necessary condition for the observability of the LDS defined by $(\wcheck{\Phi}C, A)$ is that the LDS defined by $(C, A)$ is observable.
However, for the LDSs we consider in this paper, $N \gg d$; for such systems, the LDS defined by $(C, A)$ is observable. Given this assumption, we consider the observability of the LDS parameterized by  $( \wcheck{\Phi} C, A )$ next.

\begin{lma}
For $N > d$, the LDS defined by $(\wcheck{\Phi}C, A)$ is observable, with high probability, if $\wcheck{M} \ge d$ and the entries of the matrix $\wcheck{\Phi}$ are sampled i.i.d.  from a sub-Gaussian distribution.
\end{lma}

\begin{proof}
This is established by proving that $\rank(\wcheck{\Phi}C) = d$ when $\wcheck{M} \ge d$.
Assume that $\rank(\wcheck{\Phi}C) < d$, i.e., $\exists \, \alpha \in \reals^d$
such that $\wcheck{\Phi}C\alpha = 0, \alpha \ne 0$.
Let $\phi^T$ be a row of $\wcheck{\Phi}$.
The event that $\phi^T C \alpha = 0$ is one of negligible probability when 
the elements of $\phi$ are assumed to be i.i.d. according to a sub-Gaussian distribution such as Gaussian or Bernoulli.
Hence, with high probability $\rank(\wcheck{\Phi} C) = d$ when $\wcheck{M} \ge d$.
\hfill
\end{proof}

Observability is the key criterion for recovering the state sequence from the common measurements. When the LDS associated with the common measurements is observable, we can
estimate the state sequence --- up to a linear transformation --- by factorizing the block Hankel matrix  $\textrm{Hank}(\wcheck{\bfz}_{1:T}, d)$ in (\ref{eqn:hankelmat}).
$\textrm{Hank}(\wcheck{\bfz}_{1:T}, d)$ can be written as
$$\textrm{Hank}(\wcheck{\bfz}_{1:T}, d)  =
O( \wcheck{\Phi} C, A) [ \bfx_1 \; \bfx_2 \; \cdots \; \bfx_{T-d+1} ].$$
Hence, when the observability matrix $O( \wcheck{\Phi} C, A)$ is full rank, we can recover the state sequence by factoring the Hankel matrix using the $\svd$. 
Suppose the SVD of the Hankel matrix is $\textrm{Hank}(\wcheck{\bfz}_{1:T}, d) = U S V^T$.
Then,  the estimate of the state sequence is obtained by 
\begin{equation}
[\widehat{\bfx}_{1:T-d+1}] = S_{d} V_d^T,
\label{eqn:stateEst}
\end{equation}
where $S_d$ is the diagonal matrix containing the $d$-largest singular values in $S$, and $V_d$ is the matrix composed of the right singular vectors corresponding to these singular values.
The estimate of the state sequence obtained from $\svd$ differs from its true value by a linear transformation. This is a fundamental ambiguity that stems from the lack of uniqueness in the definition of the state space (see Section \ref{sec:priorlds}).
The state sequence estimate in (\ref{eqn:stateEst}) can be improved, especially for high levels of measurement noise, by using system identification techniques mentioned in Section \ref{sec:priorlds}. However, the simplicity of this estimate makes it amenable for further analysis.

{\sloppy 
When $\wcheck{M} > d$, we can choose to factorize a smaller-sized Hankel matrix $\textrm{Hank}(\wcheck{\bfz}_{1:T}, q)$ provided $q >  d/\wcheck{M} $. Note that when $q = 1$, we do not enforce the constraints provided by the state transition model, thereby simply reducing the LDS to a linear system.
For $q > 1$, we enforce the state transition model over $q$ successive time instants; i.e., we enforce
$$\bfx_t = A \bfx_{t-1} = A^2 \bfx_{t-2} = \cdots = A^{q-1} \bfx_{t-q+1}, \,\,\, q \le t \le T.$$
Larger values of $q$ lead to smoother state sequences, since the estimates conform to the state transition model for longer durations.}

We next study the observability properties of specific classes of interesting LDSs and the conditions on $\wcheck{\Phi}$ under which the observability of $(\wcheck{\Phi} C, A)$ holds.

\subsection{Case: $\wcheck{M}=1$} \label{sec:one}
A particularly interesting scenario is when we obtain exactly one common measurement for each video frame. 
For such a scenario, $\wcheck{M}=1$ and, hence, the measurement matrix can be written as a row-vector: $\wcheck{\Phi} = {\phi}^T \in \reals^{1 \times N}$.
We now establish conditions when the observability matrix $O( {\phi}^T C, A)$ is full rank for this particular 
scenario. Let $\wcheck{\bf c} =   ({\phi}^T C)^T = C^T {\phi}$ and $B = A^T$. We seek a condition when the observability matrix, or equivalently its transpose,
\begin{equation}
\left( O(\wcheck{\bf c}^T, B^T ) \right)^T =  \left[ \wcheck{\bf c} \;\; B\wcheck{\bf c} \;\; B^2 \wcheck{\bf c} \;\; \cdots \;\; B^{d-1} \wcheck{\bf c} \right]
\label{eqn:kry}
\end{equation}
is full rank.\footnote{There is an interesting connection to Krylov-subspace methods here. \
In Krylov-subspace methods,  a low-rank approximation to a matrix $K$ is obtained by forming the matrix $[ {\bf c} \;\; K{\bf c} \;\; K^2{\bf c} \; \cdots ]$ with $\bf c$ randomly chosen.
Convergence proofs for this method are closely related to Theorem \ref{thm:single}. To the best of our knowledge, diagonalizability of $K$ plays an important role in most of these proofs. The interested reader is referred to \cite{saad1981krylov} for more details.}
We concentrate on the specific scenario where the matrix $B$ (and hence, $A$) is diagonalizable, i.e., $B = Q \Lambda Q^{-1}$, where $Q \in \reals^{d \times d}$ is an invertible matrix (hence, full rank) and $\Lambda$ is a diagonal matrix with diagonal elements $\{ \lambda_i, \; 1 \le i \le d \}$.
For such matrices, the transpose of the observability matrix can be written as
$$
\begin{array}{rcl}
\left( O(\wcheck{\bf c}^T, B^T) \right)^T 
&=& \left[ \wcheck{\bf c} \;\; B\wcheck{\bf c} \;\; B^2 \wcheck{\bf c} \;\; \cdots \;\; B^{d-1} \wcheck{\bf c} \right] \\
&=& \left[ QQ^{-1} \wcheck{\bf c} \;\; Q \Lambda Q^{-1} \wcheck{\bf c} \;\; \cdots \;\; Q \Lambda^{d-1} Q^{-1} \wcheck{\bf c} \right] \\
&=& Q \left[ {\bf e} \;\; \Lambda {\bf e} \;\; \Lambda^2 {\bf e} \;\; \cdots \;\; \Lambda^{d-1} {\bf e} \right], \\
\end{array}$$
where ${\bf e} = Q^{-1} \wcheck{\bf c}$.
This can be expanded as
$$ Q \left[
\begin{array}{ccccc}
e_1 & \lambda_1 e_1 & \lambda_1^2 e_1 & \cdots & \lambda_1^{d-1} e_1 \\
e_2 & \lambda_2 e_2 & \lambda_2^2 e_2 & \cdots & \lambda_2^{d-1} e_2\\
\vdots & \vdots & \vdots & \cdots & \vdots \\
e_d & \lambda_d e_d & \lambda_d^2 e_d & \cdots & \lambda_d^{d-1} e_d\\
\end{array}
\right]$$
and further into
$$ Q \left[
\begin{array}{cccc}
e_1 & & & 0\\
 & e_2 & & \\
 & & \ddots & \\
 0 & & & e_d
  \end{array}
\right]
\left[
\begin{array}{ccccc}
1 & \lambda_1 & \lambda_1^2 & \cdots & \lambda_1^{d-1} \\
1 & \lambda_2 & \lambda_2^2 & \cdots & \lambda_2^{d-1} \\
\vdots & \vdots & \vdots & \cdots & \vdots \\
1 & \lambda_d & \lambda_d^2 & \cdots & \lambda_d^{d-1} \\
\end{array}
\right].
$$

We can establish a sufficient condition for when the observability matrix is full rank.

\begin{thm} \label{thm:single}
Let  $\wcheck{M} = 1$ and let the elements of $\wcheck{\Phi} = {\bf \phi}^T$ be i.i.d. from a sub-Gaussian distribution. Then, with high probability, the observability matrix is full rank  when the state transition matrix is diagonalizable and its eigenvectors and eigenvalues are unique.
\end{thm}

\begin{proof}
From the discussion above, the observability matrix can be written as a product of three square matrices: $Q$, the matrix of eigenvectors of $A^T$; a diagonal matrix with entries defined by the vector ${\bf e} = Q^{-1}C^T {\bf \phi}$; and a Vandermonde matrix defined by the vector of eigenvalues of $A$ $\{\lambda_i, \; 1 \le i \le d \}$.
When the eigenvectors and eigenvalues are distinct, the first and last matrices are full rank.
Given that the elements of ${\bf \phi}$ are i.i.d., the probability that $e_i = 0$ is negligible and, hence, the diagonal matrix is full rank with high probability.
Since the product of full rank square matrices is full rank as well, this implies that the observability matrix is full rank with high probability.
\hfill
\end{proof}

{ {\bf Remark: }} Theorem \ref{thm:single} requires that the state-transition matrix be full-rank (non-zero Eigenvalues) and be diagonalizable with unique Eigenvalues. Most matrices are diagonalizable (once, we allow complex Eigenvalues) and hence, the requirement that state transition matrix be diagonalizable is not restrictive. A more restrictive condition is requiring the Eigenvalues of the matrix to be unique. Unfortunately, this eliminates some commonly observed state transition matrix such as the Identity matrix --- which is coupled with Brownian processes. 
Nonetheless, Theorem \ref{thm:single} is intriguing, since it guarantees recovery of the state sequence even when we
obtain only \emph{one} common measurement per time instant. 
This is immensely useful in reducing the number of measurements 
required to sense a video sequence.

%Similar results hold also when $A$ is an arbitrary rotation matrix \cite{wakin2010observability}. However a precise characterization is tough. \note{ACS: This is messy!}

Interestingly, we can reduce $\wcheck{M}$ even further. 
This is achieved by not obtaining common measurements at some time instants. 
%These missing measurements are first
%estimated by exploiting the low-rank property of the Hankel matrix $\textrm{Hank}(\wcheck{\bfz}_{1:T}, q)$.

\subsection{Missing measurements: Case $\wcheck{M} < 1$} \label{sec:miss}
If we do not obtain common measurements at  some time instants, then is it still possible to obtain an estimate of the state sequence?
One way to view this problem is that we have incomplete knowledge of the Hankel matrix defined in (\ref{eqn:hankelmat}) and we seek to \emph{complete} this matrix.
Matrix completion, especially for low rank matrices, has received significant attention recently
 \cite{recht2007guaranteed, candes2009exact, candes2010power}.
 
 Given that the Hankel matrix $\textrm{Hank}(\wcheck{\bfz}_{1:T}, q)$ in (\ref{eqn:hankelmat})   is low rank for videos modeled as LDSs, we formulate the missing measurement recovery problem as one of matrix completion.
%We exploit the low rank property of the Hankel matrix to complete the matrix.
%We exploit the low rank property of the matrix $\textrm{Hank}(\wcheck{\bfz}_{1:T}, d)$ to recover the missing measurements at all time instants.
%The ability to handle missing measurements plays an important role in the design of the overall CS-LDS sensing pipeline. 
Suppose that we have the common measurements only at time instants given by the index set ${\mathscr I} \subset \{1, \ldots, T\}$, i.e., we have knowledge of $\{ \wcheck{\bfz}_i, \; i \in {\mathscr I} \}$.
We can recover the missing measurements by exploiting the low-rank property of $\textrm{Hank}(\wcheck{\bfz}_{1:T}, q)$.
Specifically, we solve the following problem to obtain the missing measurements:
$$ \min \textrm{rank}(\textrm{Hank}({\bf g}_{1:T}, q)) \; \; \; \textrm{s.t.} \; \; {\bf g}_i = \wcheck{\bfz}_i, i \in {\mathscr I}.$$
However, $\textrm{rank}(\cdot)$ is a non-convex function which renders the above problem  NP-complete.
In practice, we can solve a convex relaxation of this problem\footnote{Historically, the use of nuclear norm-based optimization for system identification goes back to Fazel et al. \cite{fazel2001rank, fazel2003log}.
Since then, there has been much work towards establishing the equivalence of these two problems \cite{recht2007guaranteed, candes2009exact}. Further, the convex program in (\ref{eqn:hankConv}) was used for video inpainting in \cite{ding2007rank}.}
%In addition to this, there has been much work in developing efficient optimization tools \cite{cai2008singular, lee2010admira, mitra2010large} that
%solve the nuclear norm problem. }
\begin{equation}
 \min \|\textrm{Hank}({\bf g}_{1:T}, q) \|_* \; \; \; \textrm{s.t.} \; \; {\bf g}_i = \wcheck{\bfz}_i, i \in {\mathscr I},
 \label{eqn:hankConv}
 \end{equation}
where $\| H \|_{*}$ is the nuclear norm of the matrix $H$, which equals the sum of its singular values.
Once we fill in the missing measurements, we use (\ref{eqn:stateEst}) to recover an estimate of the state sequence.
% as in Section \ref{sec:observe}.

An important quantity to characterize is the proportion of time instants in which we can choose to not obtain common measurements. 
This amounts to developing a sampling theorem for the completion of low-rank Hankel matrices; to the best of our knowledge, there has been little theoretical work on this problem.
Instead, we address it empirically in Section \ref{sec:experiments}.

\section{Estimating the observation matrix} \label{sec:obs}
In this section, we discuss estimation of the observation matrix $C$ given the estimates of the state space sequence $\widehat{\bfx}_{1:T}$.

\subsection{Need for innovation measurements}
Given estimates of the state sequence $\widehat{\bfx}_{1:T}$, the matrix $C$ is linear in the compressive measurements which enables  a host of conventional $\ell_2$-based methods as well as $\ell_1$-based recovery algorithms to estimate $C$. 
However, recall that the $C$ is a $N \times d$ matrix and, hence, the common measurements by themselves are not enough to recover $C$, unless $\wcheck{M}$ is large.

The common measurements $\wcheck{\bfz}_{1:T}$ used in the estimation of the state sequence are measured using a time-invariant measurement
matrix $\wcheck{\Phi}$. 
A time-invariant measurement matrix, by itself, is not sufficient for estimating $C$ unless $\wcheck{M}$ is very large.
%Indeed, given a $d$-dimensional state space and a fixed $\wcheck{\Phi}$, the range of $f(\bfx) = \wcheck{\Phi}C\bfx$ is a $d$-dimensional subspace and hence, for our problem
%we can only get up to $d$ linearly independent projections of $C$, given a fixed $\wcheck{\Phi}$. 
To alleviate this problem, we take additional compressive measurements of each frame using a time-varying measurement matrix.
Let $\wtilde{\bfz}_t = \wtilde{\Phi}_t \bfy_t + \omega_t =  \wtilde{\Phi}_t C \bfx_t + \omega_t,$ where $\wtilde{\bfz}_t \in \reals^{\wtilde{M}}$ and $\wtilde{\Phi}_t \in \reals^{\wtilde{M} \times N}$ are the compressive measurements and the corresponding measurement matrix at time $t$. As mentioned earlier in Section \ref{sec:cslds}, we refer to these as  innovation measurements.
Noting that $C$ is a time-invariant parameter, we can collect innovation measurements over a period of time before reconstructing $C$. This enables a significant reduction in the number of measurements taken at each time instant. 
%Indeed, as we demonstrate later, for certain classes of LDS we can recover from a single measurement per time instant.

\subsection{Structured sparsity for $C$} \label{sec:strucSpars}
Individual frames of a video, being images, exhibit sparsity/compressibility in a certain transform bases such as wavelets and DCT. If the support of the frames are highly overlapping --- this is to be expected given the redundancies in a video --- then columns of $C$ are compressible in the same transform bases; a consequence of $C$ being a basis for the frames of the video.
Further, note that the columns of $C$ are also the top principal components and hence, capture the dominant motion patterns in the scene; when motion in the scene is spatially correlated, the columns of $C$ are compressible in wavelet/DCT basis.
For these reasons, we assume that the columns of $C$ are {\em compressible} in a wavelet/DCT basis and employ sparse priors in the recovery of the observation matrix $C$. 
We can potentially obtain an estimate of $C$ by solving the following convex program:
\begin{equation}
(P_{\ell_1}) \;\;\; \min \sum_{i=1}^d \| \Psi^T {\bf c}_i \|_1 \;\; \textrm{s.t} \;\;\; \forall t, \| \bfz_t - \Phi_t C \widehat{\bfx}_t \|_2 \le \epsilon.
\label{eqn:l1naive}
\end{equation}
 Here, we denote the columns of the matrix $C$ as ${\bf c}_i, i=1,\ldots, d$. $\Psi$ is a sparsifying basis for the columns of $C$; we have the freedom to choose different sparsifying bases for different columns of $C$.
 
 The assumption of compressibility in a transform basis was sufficient  for all the videos we test on (see Section \ref{sec:experiments}).
 However, it is entirely possible that a video is not compressible in a transform basis. There are two possible ways to address such a scenario.
First, given training data, we can use dictionary learning algorithms \cite{kreutz2003dictionary}  to learn an appropriate basis where in the columns of $C$ are sparse/compressible.
Second, in the absence of training data, we revert to $\ell_2$-based methods to recover $C$; in such cases, we would typically need more measurements to recover $C$.

However, the convex program $(P_{\ell_1})$ is not sufficient as-is to recover $C$. The reason for this stems from ambiguities in the definition of the LDS (see Section \ref{sec:priorlds}).
The use of $\svd$ for recovering the state sequence introduces an ambiguity in the estimates of the state sequence in the form of $[\widehat{\bfx}_{1:T}] \approx L^{-1} [{\bfx}_{1:T}]$, where $L$ is an invertible $d \times d$ matrix.
As a consequence, this will lead to an estimate $\widehat{C} = CL$ satisfying $\bfz = \Phi \widehat{C} \widehat{\bfx}_t = \Phi (CL)(L^{-1} \bfx_t) = \Phi C \bfx_t$.
%This ambiguity introduces additional concerns in the estimation of $C$.
Suppose the columns of $C$ are $K$-sparse (equivalently, compressible for a certain value of $K$) each in $\Psi$ with support ${\mathcal S}_k$ for the $k$-th column. Then, the columns of $CL$ are potentially $dK$-sparse with identical supports ${\mathcal S} = \bigcup_k {\mathcal S}_k$.
The support is exactly $dK$-sparse when the ${\mathcal S}_k$ are disjoint and $L$ is dense.
At first glance, this seems to be a significant drawback, since the overall sparsity of $\widehat{C}$ has increased to $d^2K$ (the sparsity of $C$ is $dK$).
However, this apparent increase in sparsity is alleviated by the columns having identical supports,
which can be exploited in the recovery process \cite{duarte2013measurement}.

%\subsection{Solving for $C$}
Given the estimates $\widehat{\bfx}_{1:T}$, we estimate the matrix $C$ by solving the following convex program:
\begin{equation}
(P_{\ell_2-\ell_1}) \,\, \min \sum_{i=1}^N \| \bfs_i \|_2 \;\;\; \textrm{ s.t } C = \Psi S, \; \forall t,  \| {\bfz}_t - {\Phi}_t C \widehat{\bfx}_t \|_2 \le \epsilon,
\label{eqn:probP1}
\end{equation}
where $\bfs_i$ is the $i$-th row of the matrix $S = \Psi^T C$ and $\Psi$ is a sparsifying basis for the columns of $C$.
The above problem is an instance of an $\ell_2-\ell_1$ mixed-norm optimization that promotes group sparsity; in this instance, we use it to promote group column sparsity in the matrix $S$, i.e., all columns have the same sparsity pattern.

There are multiple efficient ways to solve $(P_{\ell_2-\ell_1}),$ including  solvers such as {\em SPG-L1} \cite{berg2008probing}  and model-based CoSAMP \cite{baraniuk2008model}.
Algorithm \ref{alg:cosamp} summarizes a model-based CoSAMP algorithm used for recovering the observation matrix $C$.
The specific model used here is a union-of-subspaces model that groups each row of $S = \Psi^T C$ into a single subspace/model.

%This greedy solution offers a computationally efficient alternative to the convex program $(P_{\ell_2-\ell_1})$ at a small price in the accuracy of the result. In addition to this, in many applications, the parameters associated with the CoSAMP algorithm  are far more intuitive. Specifically, the only parameter required in Algorithm $\ref{alg:cosamp}$ is the sparsity $K$ or the expected number of non-zeros in each column of $S = \Psi^T C$.

\begin{algorithm}[!hhh]
\caption{
$\widehat{C} = $ {Model-based CoSAMP} $(\Psi, K, \bfz_t, \widehat{\bfx}_t, \Phi_t, t=1,\ldots, T)$ }
%{\bf Input:}\\
%$\Phi_t$: Compressive measurement matrix at time $t$ \\
%$\bfz_t$: Compressive measurements \hspace{23mm}$\bfx_t$: State \\
%$\Psi$: Sparsifying basis \hspace{37mm} $K$: Sparsity level\\

{ Notation:}\\
\hspace{5mm} $\textrm{supp}( vec; K )$ returns the support of $K$ largest elements of $vec$\\
\hspace{5mm} $A_{|\Omega, \cdot}$ represents the submatrix of $A$ with rows indexed by $\Omega$ and all columns.\\
\hspace{5mm} $A_{|\cdot, \Omega}$ represents the submatrix of $A$ with columns indexed by $\Omega$ and all rows.\\

{ Initialization} \\
\hspace{5mm} $\forall t, \Theta_t  \leftarrow \Phi_t \Psi$\\
\hspace{5mm}$\forall t, \bfv_t \leftarrow {\bf 0} \in {\mathbb R}^{M}$ \\
\hspace{5mm} $\Omega_{\textrm{old}} \leftarrow \phi$ \\

\While{(stopping conditions are not met) }{
 Compute signal proxy: \\
\hspace{5mm} $R = \sum_{t} \Theta_t^T \bfv_t \widehat{\bfx}_t^T$\\
Compute energy in each row:\\
\hspace{5mm} $k \in [1, \ldots, N], \bfr(k) = \sum_{i=1}^d R^2(k, i)$\\
Support identification and merger:\\
\hspace{5mm} $\Omega \leftarrow \Omega_\textrm{old} \bigcup \textrm{supp}(\bfr; {2K})$ \\
Least squares estimation:\\
\hspace{5mm} Find $A \in {\mathbb R}^{|\Omega| \times d}$ that minimizes $\sum_t \| \bfz_t - (\Theta_t)_{|\cdot, \Omega} A \widehat{\bfx}_t \|_2 $\\
\hspace{5mm} $B_{|\Omega, \cdot} \leftarrow A, \;B_{|\Omega^c, \cdot } \leftarrow 0$ \\
Pruning support:\\
\hspace{5mm} $k \in [1, \ldots, N], {\bf b}(k) = \sum_{i=1}^d B^2(k, i)$\\
\hspace{5mm} $\Omega \leftarrow \textrm{supp}( {\bf b}; K ), \; S_{|\Omega, \cdot} \leftarrow B_{|\Omega, \cdot}, \; S_{|\Omega^c, \cdot } \leftarrow 0$ \\
Form new estimate of $C$:\\
\hspace{5mm} $\widehat{C} \leftarrow \Psi S$\\
Update residue:\\
\hspace{5mm} $\forall t, \bfv_t \leftarrow \bfz_t - \Theta_t S \widehat{\bfx}_t$\\
\hspace{5mm} $\Omega_\textrm{old} \leftarrow \Omega$ \\
}
%\caption{{Pseudo-code of the model-based CoSAMP algorithm for CS-LDS.}}
\label{alg:cosamp}
\end{algorithm}

\subsection{Value of $\wtilde{M}$} \label{sec:numbah}
For stable recovery of the observation matrix $C$, we need in total $O(dK \log(N/K))$ measurements; for a large class of practical solvers, a rule of thumb is $4 dK \log(N/K)$. Given that we measure $\wtilde{M}$ time-varying compressive measurements at each time instant, over a period of $T$ time instants, we have $\wtilde{M}T$ compressive measurements for estimating $C$. Hence, for stable recovery of $C$, we need approximately
\begin{equation}
\wtilde{M}T = 4 dK \log(N/K) \implies \wtilde{M} = 4 \frac{dK}{T}\log(N/K).
\label{eqn:measBounds}
\end{equation}
This indicates extremely favorable operating scenarios for the CS-LDS framework, especially when $T$ is large (as in high frame rate capture).
Let $T = \tau f_s,$ where $\tau$ is the time duration of the video in seconds and $f_s$ is the sampling rate of the measurement device. 
The number of compressive measurements required in this case is $\wtilde{M} = 4 \frac{d K}{\tau f_s } \log \left(\frac{N}{K} \right)$.
Given that the complexity of the LDS typically (however, not always) depends on $\tau$, for a fixed $\tau$ the number of measurements required to estimate $C$ decreases as $1/f_s$ as the sampling rate $f_s$ is increased.
Indeed, as the sampling rate $f_s$  increases, $\widetilde{M}$ can be decreased while keeping $Mf_s$ constant. This will ensure that (\ref{eqn:measBounds}) is satisfied, enabling stable recovery of $C$. 
%This is a very desirable property, especially for high-speed imaging.

\subsection{Mean + LDS} \label{sec:meanLDS}
In many instances, a dynamical scene is modeled better as an LDS over a static background, that is, $\bfy_t = C \bfx_t + {\bf \mu}$. This can be handled with two small modifications to the Algorithm \ref{alg:cosamp}. First, the state sequence $[ \hat{\bfx}_{1:T}]$ is obtained by performing an SVD on the matrix $\textrm{Hank}(\widecheck{z}_{1:T}, d_\textrm{guess})$ modified such that each row sums to zero. This works under the assumption that the sample mean of $\widecheck{z}_{1:T}$ is equal to $\widecheck{\Phi} {\bf \mu}$, the compressive measurement of ${\bf \mu}$. Second, given that the support of ${\bf \mu}$ need not be similar to that of $C$,  the resulting optimization problem can be reformulated as
\begin{equation}
(P_{\mu, \ell_2-\ell_1}) \,\, \min \| \Psi^T \mu \|_1+\sum_{i=1}^N \| \bfs_i \|_2 \;\;\; \textrm{ s.t } C = \Psi S, \; \forall t,  \| \wtilde{\bfz}_t -  \wtilde{\Phi}_t (\mu +  C \widehat{\bfx}_t) \|_2 \le \epsilon.
\label{eqn:probP1mu}
\end{equation}
As with the convex formulation, the model-based CoSAMP algorithm described in Algorithm~\ref{alg:cosamp} can be modified to incorporate the mean term ${\bf \mu}$; an additional
modification here is the requirement to specify a priori the sparsity of the mean $K_\mu = \| \Psi^T \mu \|_0$. 
% As in the convex formulation above, the mean is treated separately and hence, it is also required to specify a priori its sparsity $K_\mu = \| \Psi^T \mu \|_0$. 
%We describe this modified algorithm in Appendix B.

\section{Experiments} \label{sec:experiments}
We present a range of experiments validating various aspects of the CS-LDS framework. 
%For most experiments, we chose $\widecheck{M} = 3d$, with $d$ and $K$ chosen appropriately. 
We use permuted noiselets \cite{coifman2001noiselets} for the measurement matrices,
since they have a fast scalable implementation.
%%as well as lead to measurement vectors that can be easily implemented on an SPC.
We use the term \emph{compression ratio} $N/M$ to denote the reduction in the number of measurements as compared to the Nyquist rate.
%The reciprocal of the compression ratio is the {\em measurement rate} $M/N$ which denotes the number of measurements acquired per time instant as a fraction of the Nyquist rate.
Finally, we use the reconstruction SNR  to evaluate the recovered videos.
Given the ground truth video $\bfy_{1:T}$ and a reconstruction $\widehat{\bfy}_{1:T}$, the reconstruction SNR in dB is defined by
\begin{equation}
10 \log_{10} \left( \frac{\sum_{t=1}^T \| \bfy_t \|_2^2}{\sum_{t=1}^T \| \bfy_t - \widehat{\bfy}_t \|_2^2} \right).
\label{eqn:snr_def}
\end{equation}

We compare CS-LDS against  {\em frame-by-frame} CS, where each frame of the video is recovered separately using conventional CS techniques.
We use the term {\em oracle LDS} when the parameters and video reconstruction are obtained by operating on the original data itself.
Oracle LDS estimates the parameters using a rank-$d$ approximation of the ground truth data. The reconstruction SNR of the oracle LDS gives an upper bound on the achievable SNR. Finally, the ambiguity in the observation matrix (due to non-uniqueness of the SVD based factorization) as estimated by oracle LDS and CS-LDS is resolved by finding the best $d \times d$ linear transformation that registers the two estimates.

\begin{figure*}[!ttt]
\center
\begin{tabular}{cc}
\includegraphics[width=0.45\textwidth]{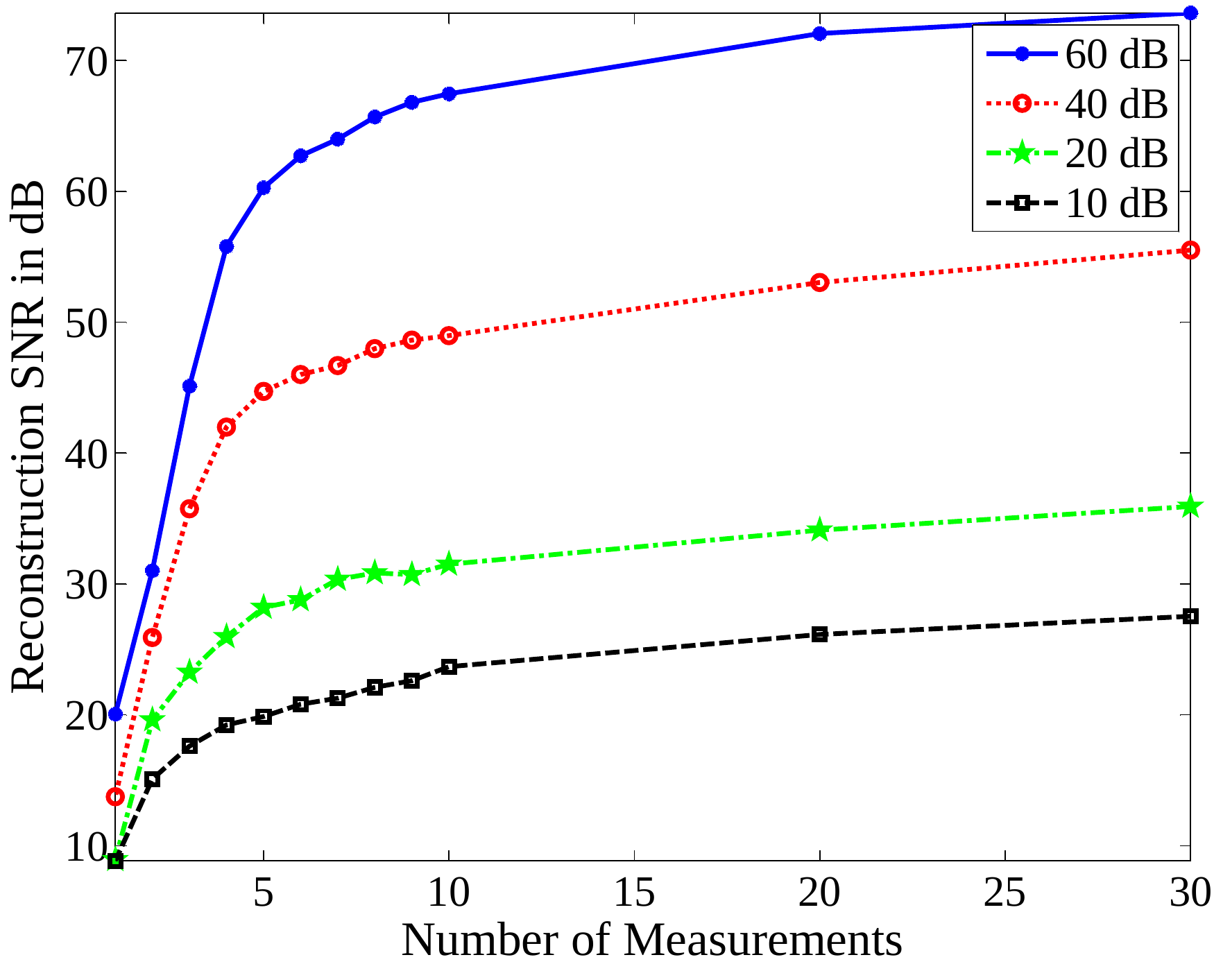} &
\includegraphics[width=0.45\textwidth]{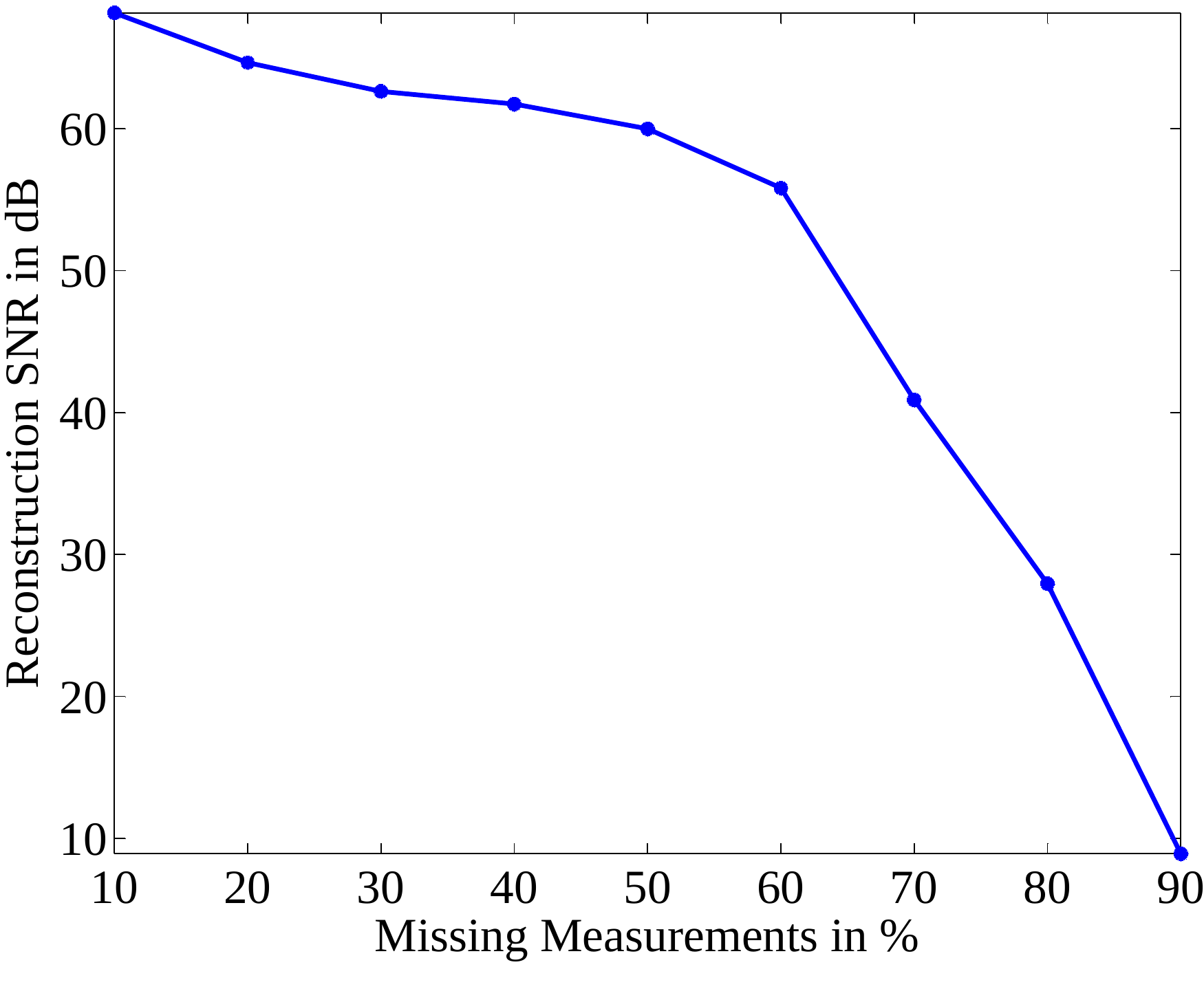} \\
(a)  $\wcheck{M} \ge 1$ & (b) $\wcheck{M} < 1$ \\
\end{tabular}
\caption{Accuracy of state sequence estimation from common measurements. Shown are aggregate results over $100$ Monte-Carlo runs for an LDS with $d=10$ and $T=500$. For each Monte-Carlo run, the system matrices and the state sequence were generated randomly. (a) Reconstruction SNR as a function of the number of common measurements $\wcheck{M}$ per frame. Each curve is for a different level of measurement noise as measured using input SNR.
For low  noise levels, we obtain a good reconstruction SNR ( $ > $ 20 dB) even at $\wcheck{M} = 1$; this 
hints at very high compression ratios.
(b) Reconstruction SNR of the Hankel matrix for the scenario with missing common measurements.  We can estimate the Hankel matrix very accurately even at $80 \%$ missing measurements. This suggests  immense flexibility in the implementation of the CS-LDS system.}
\label{fig:state}
\end{figure*}

\subsection{State sequence estimation}
We first provide empirical verification of the results derived in Sections \ref{sec:observe} and \ref{sec:one}.
It is worth noting that, in the absence of noise, Theorem \ref{thm:single} suggests exact recovery of the state sequence. In practice, it is important to check the robustness of the estimate to measurement noise. 
Figure \ref{fig:state}(a) analyzes the performance of the state space estimation for different values of the number of common measurements $\wcheck{M}$ and different SNRs of the measurement noise. We define input SNR in dB as $10 \log_{10} \left( (\sum \|\bfy_t \|_2^2)/(T \sigma^2) \right)$, where $\sigma$ is the standard deviation of the noise.
Here, we consider the scenario when $\wcheck{M} \ge 1$.
The underlying state space dimension is $d=10$ with $T=500$ frames.
As expected, for low SNRs, the reconstruction SNR is very high even for small values of $\wcheck{M}$. In addition to this, the accuracy at $\wcheck{M}=1$ is acceptable, especially at low SNRs.

Next, we validate the implications of Section \ref{sec:miss}, where we discuss the scenario of $\wcheck{M} < 1$ by simulating various proportions of missing common measurements. Figure \ref{fig:state}(b) shows reconstruction SNR for the Hankel matrix in (\ref{eqn:hankelmat}) for varying amounts of missing measurements. 
We recover the Hankel matrix by solving (\ref{eqn:hankConv}) using CVX \cite{grant2011cvx}.
Figure \ref{fig:state}(b) demonstrates a very high reconstruction SNR even at a very high rate of missing measurements. As mentioned earlier, not having to sense common measurements at all frames is very useful, since we can stagger our acquisition of common and innovation measurements. 
In theory, this enables a measurement strategy where we need to sense only one measurement per frame of the video without having to group consecutive measurements of the SPC. Hence, we can aim to reconstruct videos at the sampling rate of the SPC. To the best of our knowledge, this is the first video CS acquisition design capable of doing this.

\begin{figure*}[!ttt]
\center
\includegraphics[width=\textwidth]{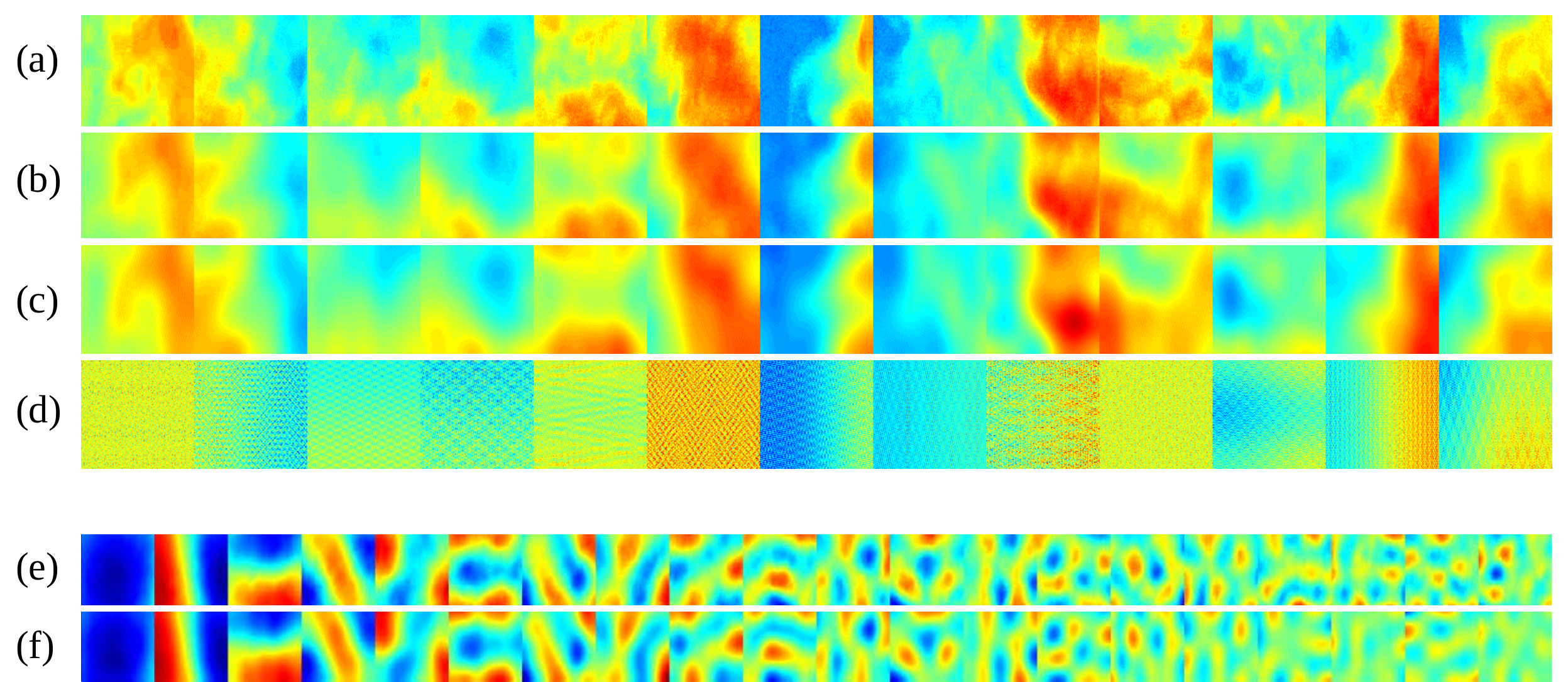}
\caption[Fire reconstruction]{Reconstruction of a fire texture of length $250$ frames and resolution of $N = 128 \times 128$ pixels.
(a-d) Sampling of frames of the (a) Ground truth video, (b) Oracle LDS reconstruction, (c) CS-LDS reconstruction, and (d) naive frame-to-frame CS reconstruction.
 The CS-LDS reconstruction closely resembles the oracle LDS result.
For the CS-LDS results, compressive measurements were obtained at $\widecheck{M}=30$ and $\widetilde{M}=40$ measurements per frame, there by giving a compression ratio of $234\times.$
 Reconstruction was performed with  $d=20$ and $K=30$. 
% The frame-to-frame CS result was obtained at the same compression.
 (e) Ground truth observation matrix $C$. (f) CS-LDS estimate of the observation matrix $\widehat{C}$. In (e) and (f), the column of the observation matrix is visualized as an image.
Both the frames of the videos and the observation matrices are shown in false-color for better contrast.
}
\label{fig:fireFrames}
\end{figure*}

\subsection{Dynamic Textures}
Our test dataset comprises of videos from the DynTex dataset \cite{dynTex}. 
We used the mean+LDS model from Section \ref{sec:meanLDS} for all the video CS experiments with the 2D DCT as the sparsifying basis for the columns of $C$ and 2D wavelets as the sparsifying basis for the mean.
We used the model-based CoSAMP solver in Algorithm \ref{alg:cosamp} for these results, since it provides explicit control of the sparsity of the mean and the columns of $C$. We used (\ref{eqn:measBounds}) as 
a guide to select these values.

Figure \ref{fig:fireFrames} shows video reconstruction of a dynamic texture from the DynTex dataset \cite{dynTex}. Reconstruction results are under a compression $N/M =  234;$ this is an operating point where frame-to-frame CS recovery is completely infeasible. However, the dynamic component of the scene is relatively small ($d=20$), which allows us to recover the video from relatively few measurements. The reconstruction SNRs of the recovered videos shown are as follows: oracle LDS = $24.97$ dB, frame-to-frame CS  = $11.75$ dB and CS-LDS = $22.08$ dB.

\begin{figure}[!ttt]
\center
\includegraphics[width=\textwidth]{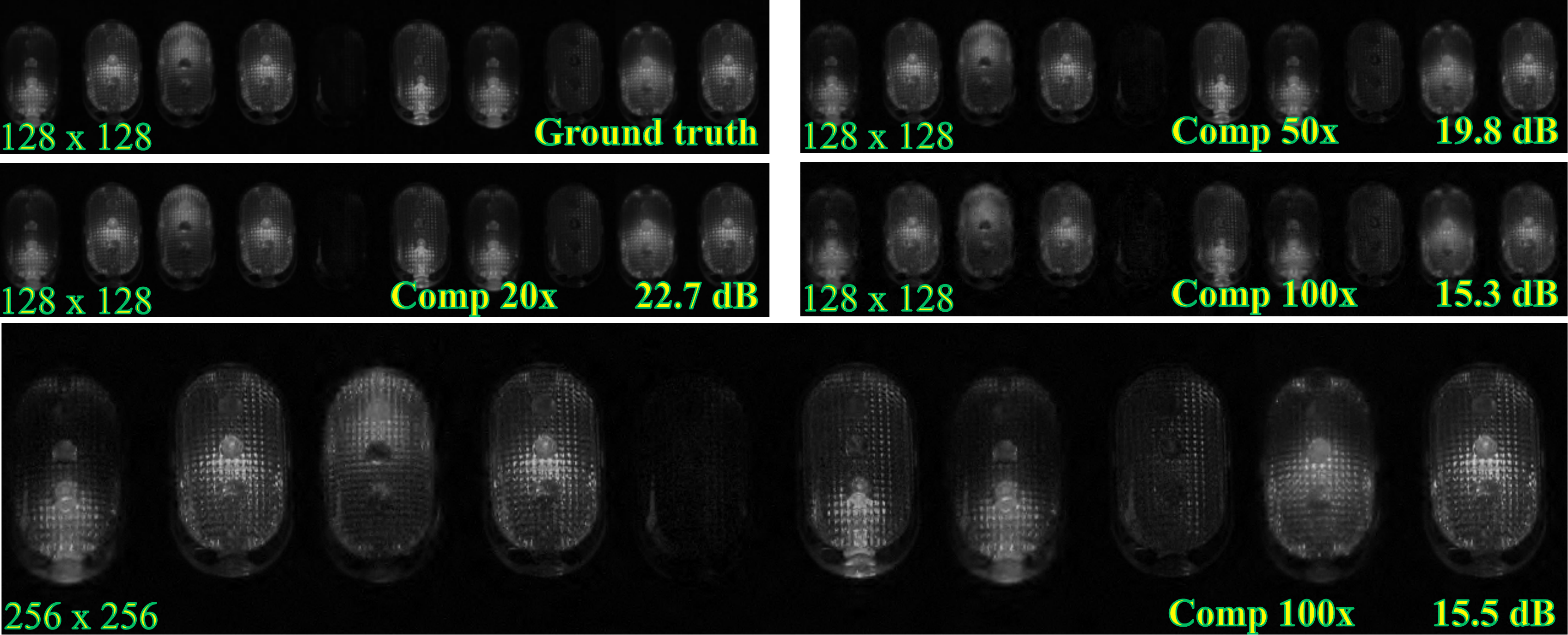}
\caption{Reconstruction of a video comprising of 6 blinking LED lights. We used $d = 7$, $\wcheck{M}=3d$, and $\wtilde{M}$ chosen based on the overall compression ratio $N/(\wcheck{M}+\wtilde{M})$. 
Each row shows a sampling of frames of the video reconstructed at a different compression ratios.
Inset in each row is the resolution of the video used as well as the compression at sensing and the reconstruction SNR.
While performance degrades with increasing compression, it also gains significantly for higher dimensional data; the reconstruction at $256 \times 256$ pixels preserves finer details.}
\label{fig:LED}
\end{figure}

Figure \ref{fig:LED} shows the reconstruction of a video, of 6 blinking LED lights, from the DynTex dataset. We show reconstruction results at different compression
ratios as well as different image resolutions. It is noteworthy that, even at  a $100\times$ compression, the reconstruction
at a resolution of $256 \times 256$ pixels preserves fine details.

\begin{figure*}[!ttt]
\center
\includegraphics[width=\textwidth]{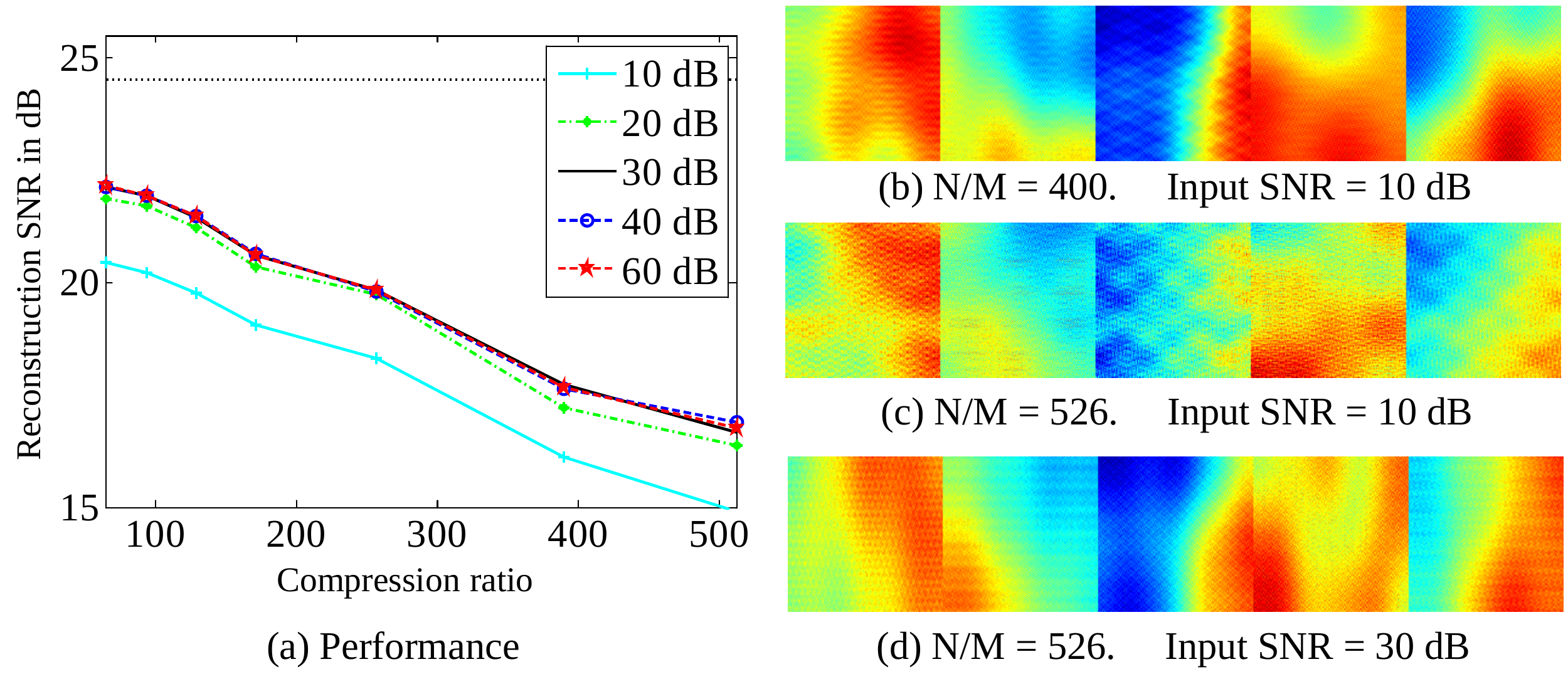}
\caption[Fire noise]{Resilience of the CS-LDS framework to measurement noise.  (a) Performance plot charting the reconstruction SNR as a function of compression ratio $N/M$. Each curve is for a different level of measurement noise as measured using and input SNR. Reconstruction SNRs were computed using 32 Monte-Carlo simulations. The ``black-dotted'' line shows the reconstruction SNR for an $d=20$ oracle LDS. (b-d) Snapshots of video frames at various operating points. The dynamic texture of Fig.\ \ref{fig:fireFrames} was used for this result.}
\label{fig:fireLDSNoise}
\end{figure*}

{\flushleft \textbf{Performance with measurement noise: }}
We validate the performance of our recovery algorithm under various amounts of measurement noise. Note that the columns of $C$ with larger singular values are, inherently, better conditioned to deal with this measurement error. The columns corresponding to the smaller singular values are invariably estimated with higher error. Figure \ref{fig:fireLDSNoise} shows the performance of the recovery algorithm for various levels of measurement noise. 
 The effect of the measurement noise on the reconstructions is perceived only at low input SNRs.
 In part, this robustness to measurement noise is due to the LDS model mismatch dominating
 the reconstruction error at high input SNRs. As the input SNR drops significantly below the model mismatch term, predictably, it starts influencing the reconstructions more. This provides a certain amount of flexibility in the design of potential CS-LDS cameras.
 
 {\flushleft \textbf{Computation time and spatial resolution: }} Figure \ref{fig:resolution} shows recovery algorithm applied to a video of length $560$ frames at different spatial resolutions. 
Shown in Figure \ref{fig:resolution} are the amount of time taken for each recovery, which scales gracefully for increasing spatial resolution, and reconstruction SNR, which approaches the performance of an oracle LDS.
The improvement in reconstruction comes due to the increase in the number of compressive measurements at high resolutions, since the compression ratio is held fixed.
However this does comes at the cost of requiring a faster compressive camera to acquire the data since a larger number of measurements.

%
%\textbf{Sampling rate: } As mentioned in Section \ref{sec:numbah}, the number of measurements required to recover the observation matrix $C$ depends on the length of the video sequence. 
%We now consider an experiment where we simulate different sampling rates 
%
%Figure \ref{fig:candleSampling} shows reconstruction plots of a candle sequence; the ground truth data was captured using a high-speed camera operating at $1024 hz$. 
%We reconstruct this sequence at different sampling rates to demonstrate the relationship between achievable compression and sampling rate of the video.
%As we vary the sampling rate of the video, we predict the number of measurements required using (\ref{eqn:measBounds}).
%
%As expected, the reconstruction SNR remains the same, while the measurement rate decreases significantly with a linear increase in the sampling rate. This makes the CS-LDS framework extremely promising for high speed capture applications.
%In contrast, most existing video CS algorithms have measurement rates that, at best, remain constant as the sampling rate increases.
%
%\begin{figure}[!ttt]
%\center
%\includegraphics[width=0.48\textwidth]{img/Flame_fps.pdf}
%\caption[Time freq]{\small{As the sampling frequency $f_s$ increases, we maintain the same reconstruction capabilities for significantly lesser number of measurements. Shown are reconstructions for $N=64 \times64$  and various sampling frequencies, achieved measurement rates, and reconstruction SNRs.}}
%\label{fig:candleSampling}
%\end{figure}

\begin{figure*}[!ttt]
\includegraphics[width=\textwidth]{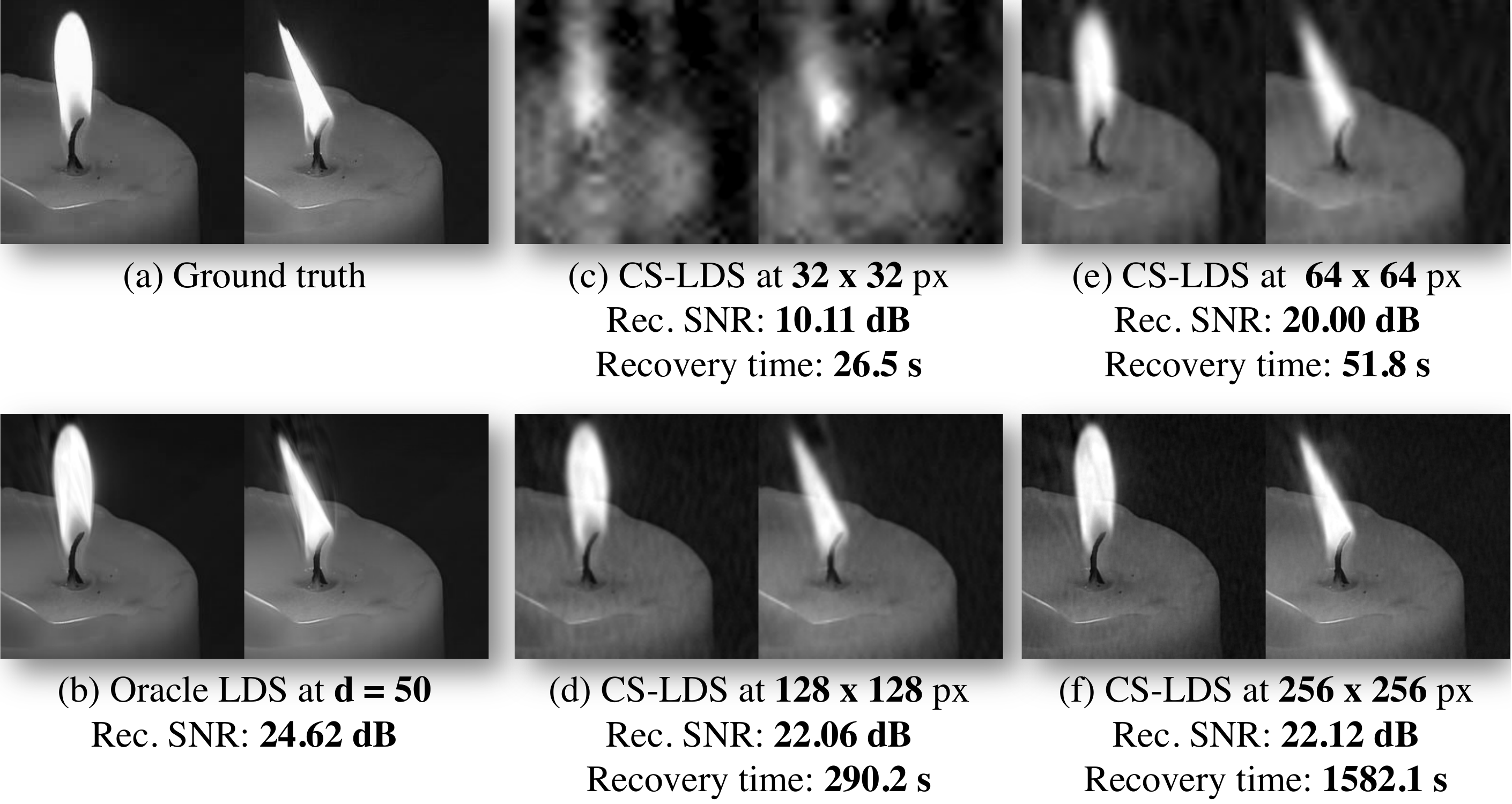}
\caption{Reconstruction of a video at different spatial resolutions. (a) Two frames from the ground truth video of $560$ frames. (b) Reconstructed frame using an oracle with $d = 50$. (c-f) CS-LDS reconstructions for varying spatial resolution, at a compression of $20\times$ and with $d = 50$. Shown are reconstruction SNR as well as recovery times for each reconstruction.  Note that as the spatial resolution increases, the reconstruction performance increases and approaches the performance of oracle LDS. However, for the same compression, recovering at a higher resolution also requires a compressive camera capable of sampling faster.}
\label{fig:resolution}
\end{figure*}

\begin{figure*}[!hhh]
\center
\includegraphics[width=\textwidth]{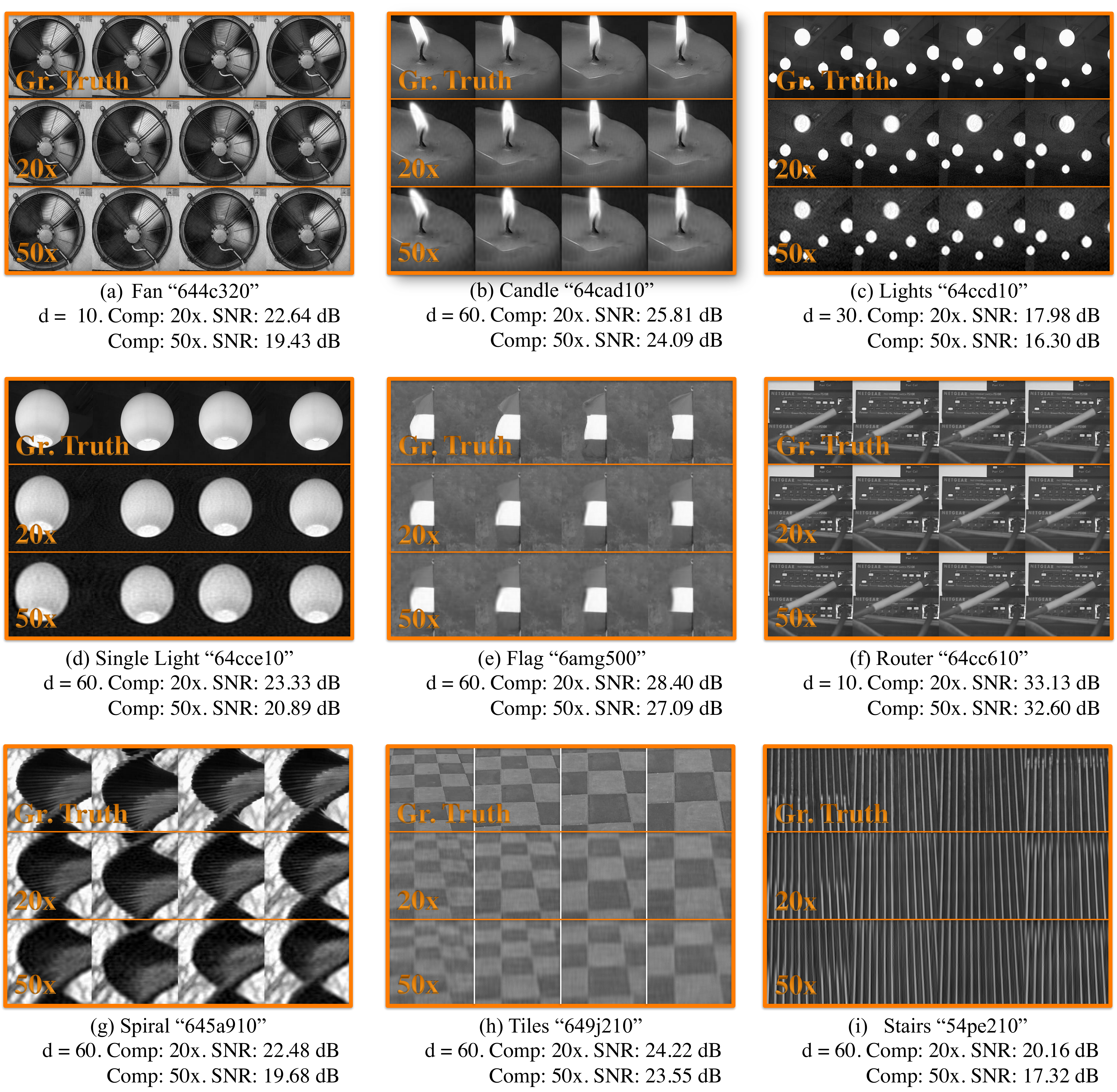}
\caption{A gallery of reconstruction results using the CS-LDS framework. 
Each sub-figure (a-i) shows reconstruction results for a different video.
The three rows of each sub-figure correspond to, from top-bottom, the ground truth video and  CS-LDS reconstructions at compression ratios of $20\times$ and $50\times$. 
Each column is a frame of the video and its reconstruction.
Also noted with each reconstruction is the value of $d$ and the reconstruction SNR for that result.
All videos are from the 
DynTex dataset \cite{dynTex} downsampled at a spatial resolution of $256 \times 256$ pixels. The ``code'' in quotes refer to the name of the sequence in the database.
For all videos, $\wcheck{M}=3d$. 
Results are best viewed under the ``zoom'' tool.
The interested reader is directed to the project webpage \cite{cslds} and the supplemental material for videos of these results.
}
\label{fig:amalgam2}
\end{figure*}

{\flushleft \textbf{Gallery of results:}}  Finally, in Figure \ref{fig:amalgam2}, we demonstrate performance of the CS-LDS methodology for sensing and reconstructing
a wide range of videos. The reader is directed to the supplemental material as well as the project webpage \cite{cslds} for animated videos of these results.

\subsection{Application in activity analysis}
As mentioned in Section \ref{sec:priorlds}, LDSs are often used in classification problems, especially in the context of scene/activity analysis.
A key experiment in this context is to check if the CS-LDS framework recovers videos 
that are sufficiently informative for such applications.
To this end, we experiment with two different activity analysis datasets: the UCSD Traffic Dataset \cite{Chan2005} and the UMD Human Activity Dataset \cite{veeraraghavan2006function}.

\paragraph{Activity recognition methodology} In both the scenarios considered here (single human activity, and traffic), we model the observed video using the linear dynamical model framework. For recognition, we used the Procrustes distance \cite{chikuse2003statistics} between the column spaces of the observability matrices in conjunction with a nearest-neighbor classifier. Given the observability matrix $O(C, A)$ defined in (\ref{eqn:obsMat}), 
%\begin{align}
%O_{n}^T(M) = \left[C^T, (CA)^T, (CA^2)^T, \ldots (CA^n)^T\right]
%end{align}
let $Q$ be an orthonormal matrix such that ${\rm span}(Q) = {\rm span}(O(C, A))$. 
Given two LDSs, the squared Procrustes distance  between them is given by
\begin{align*}
d^2(Q_1, Q_2) &= \min_{R \in \reals^{d \times d}} {\rm tr}(Q_1 - Q_2R)^T(Q_1 - Q_2R),
%&= \min_{R \in \reals^{d \times d}} {\rm tr}(R^TR - 2Q_1^TQ_2R + I_k) ,
\end{align*}
where ${\rm span}(Q_1) = {\rm span}(O(C_1, A_1))$ and ${\rm span}(Q_2) = {\rm span}(O(C_2, A_2))$.
%Since $R$ varies over the space of all $d \times d$ matrices, the minimum is attained at $R = A$. 
We use this distance function in a nearest neighbor classifier in both the activity classification experiment.

{\flushleft {\bf The UCSD Traffic Dataset}} \cite{Chan2005} consists of $254$ videos capturing traffic of three types: light, moderate, and heavy. 
Each video is of length $50$ frames at a resolution of $64 \times 64$ pixels.
Figure \ref{fig:traffic} shows the reconstruction results on a traffic sequence from the dataset. 
We perform a classification experiment of the videos into these three categories. There are four different train-test scenarios provided with the dataset. For comparison, we also perform the same experiments with fitting the LDS model on the original frames (oracle LDS). We perform classification at two different values of the state space dimension $d$ and at a fixed compression ratio of $25 \times$. Table \ref{tab:trafficClass} shows classification results. We also show comparative results obtained using a probabilistic kernel on dynamic texture models \cite{Chan2005} in conjunction with SVMs in the last two rows of the table. Results for each individual experiment were not reported, only an aggregate number was reported which is shown in the table. It can be seen that even without sophisticated non-linear classifiers, we are able to obtain comparable performance using a simple nearest neighbor classifier using the dynamic texture model parameters. This shows that the obtained parameters possess discriminatory properties, and can be used in conjunction with other sophisticated classifiers that build on dynamic texture models as in \cite{Chan2005}.

{\flushleft {\bf The UMD Human Activity Dataset}} \cite{veeraraghavan2006function} consists of $100$ videos, each of length $80$ frames, depicting $10$ different activities: {\it pickup object, jog, push, squat, wave, kick, bend, throw, turn around and talk on cellhpone.} Each activity was repeated $10$ times, so there were a total of $100$ sequences in the dataset.
As with the traffic experiment, we use an LDS model on the image intensity values without any  feature extraction. 
%Given the constrained nature of the database, the LDS model works well and serves to test the efficacy of the model reconstruction algorithm.  
Images were cropped to contain the human and resized to $330 \times 300$. The state space dimension was fixed at $d=5$ and the compression was varied from $50 \times$ to $200 \times$. We performed a leave-one-execution-out test. The results are summarized in table \ref{tab:recognition_accuracy}. As can be seen, the CS-LDS framework obtained a  classification performance that is comparable to the oracle LDS. For this dataset, both oracle LDS and CS-LDS obtained a perfect classification score of $100 \%$ up to a compression ratio of $50 \times$. Further, as shown in Table \ref{tab:recognition_accuracy}, we obtain comparable performance to a far more sophisticated method employing advance shape-based features for activity recognition. This suggests that the CS-LDS framework should be extremely useful in a wide range of applications beyond just video recovery, and can provide a basis to acquire more sophisticated features for tackling challenging activity recognition problems.

\begin{figure*}[!ttt]
\center
\includegraphics[width=\textwidth]{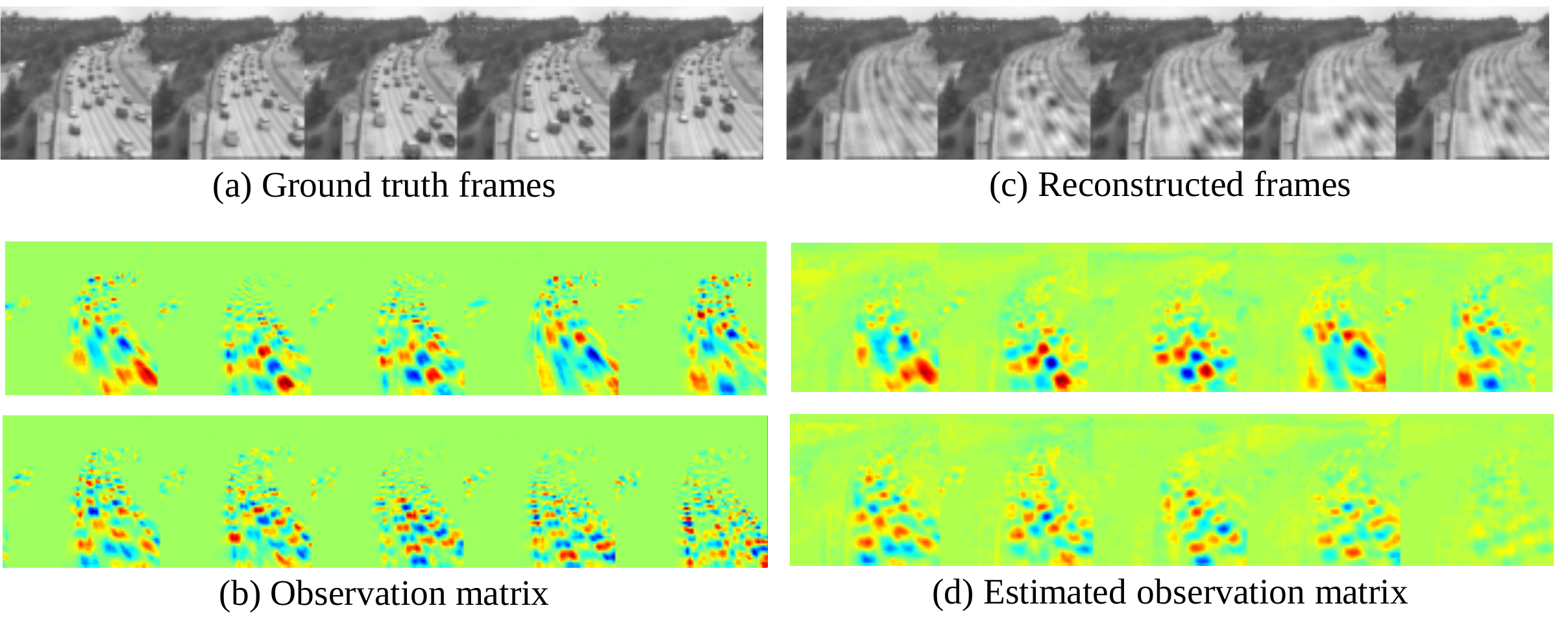}
\caption{{Reconstructions of a traffic scene of $N=64\times64$ pixels at a compression ratio $N/M = 25$, with $d=15$ and $K=40$. (a, c) Sampling of the frames of the ground truth and reconstructed video. (b, d) The first ten columns of the observation matrix $C$ and the estimated matrix $\widehat{C}$; both are shown in false color for improved contrast. The quality of reconstruction and LDS parameters is sufficient for capturing the flow of traffic as seen in the classification results in Table \ref{tab:trafficClass}. }}
\label{fig:traffic}
\end{figure*} 

\begin{table}[!ttt]
\center
\caption[Classification results]{\small{Classification results (in $\%$) on the UCSD Traffic Dataset}}
\begin{tabular}{|c||c|c|c|c|c|}
\hline 
        & Expt 1 & Expt 2 & Expt 3 & Expt 4 & Average\\
        \hline \hline
\textbf{(d = 10)} & &&& &\\
Oracle LDS    & 85.71   &  85.93 &  87.5 &   92.06 & 87.8\%\\
CS-LDS & 84.12  &  87.5  &  89.06&    85.71 & 86.59\%\\
\hline

%\begin{tabular}{|c|c|c|c|c|}
\hline

\textbf{(d = 5)} & &&& & \\
%        & Expt 1 & Expt 2 & Expt 3 & Expt 4 \\
%        \hline
Oracle LDS    & 77.77 &  82.81 & 92.18 &  80.95 & 83.42\%\\
CS-LDS & 85.71  &  73.43 &  78.1 &    76.1  & 78.34\%\\
\hline

State KL-SVM (d = 10)\cite{Chan2005} & n.a. & n.a. & n.a. & n.a. & 93\%\\
State KL-SVM (d = 5)\cite{Chan2005} & n.a. & n.a. & n.a. & n.a. & 87\%\\
\hline
\end{tabular}

\label{tab:trafficClass}
\end{table}

\begin{table}[!ttt]
\center
\caption{\small{Classification results (in $\%$) on the UMD Human Activity Database}}
\begin{tabular}{|l || c | c|c|c|}%{|p{1in}||p{0.5in}|p{0.5in}|p{0.5in}|}
\hline 
Activity &$100 \times$ &$150 \times$&$200 \times$ & Shape dynamics \cite{Veeraraghavan2005}\\
\hline\hline
Pickup Object  & 100 & 100 & 100 & 100\\
Jog  & 100 &100 & 90 & 100\\
Push  & 100 & 90 & 50 & 100\\
Squat &90 &100 & 100 & 100\\
Wave &100&  100 & 60 & 100\\
Kick  &100& 90 & 80 & 100\\
Bend  & 100& 100 & 100 & 100\\
Throw &100& 100 & 90 & 100\\
Turn Around &100& 100 &100 & 100\\
Talk on Cellphone & 100& 20 & 10 & 100\\
\hline
Average &94\% & 90\% & 78\% & 100\%\\
\hline
\end{tabular}
\label{tab:recognition_accuracy}
\end{table}

\section{Discussion} \label{sec:discuss}
In this paper, we have proposed a framework for the compressive acquisition of dynamic scenes modeled as LDSs. In particular, this paper emphasizes the power of predictive/generative video models. In this regard, we have shown that a strong model for the scene dynamics enables stable video reconstructions at very low measurement rates. 
In particular, it enables the estimation of the state sequence associated with a video even at fractional number of common measurements per video frame ($\wcheck{M} \le 1$).
The use of CS-LDS for dynamic scene modeling and classification also highlights the purposive nature of the  framework.
%We conclude the paper with a discussion of limitations of the existing model, as well as future 

{\flushleft {\bf Implementation issues: }}
The results provided in the paper are mainly based on simulations. While a full-fledged implementation on hardware is beyond the scope of this paper, we discuss some of the key issues and challenges in obtaining such results. Focusing on the single pixel camera (SPC) as our imaging architecture, the achievable compression and resolution are limited by the amount of motion in the scene and the sampling rate of the camera.
We discuss the roles these two parameters play in practice.

Amount of motion determines an inherent notion of frame-rate of the video; note that real life scenes have no notion of ``frame-rate''. If the scene changes negligibly for a time duration $\tau$, then $1/\tau$ (for the largest value of $\tau$) becomes a good measure of frame-rate for a scene.  For example, static scenes do not change over an infinite time duration ($\tau = \infty$) and hence, can be sensed at $1/\tau = 0$ fps.
Given that we seek to sense this scene at a spatial resolution of $N$ pixels, a Nyquist camera would need to operate at $N/\tau$ measurements per second.

Suppose this scene over a duration of $T$ seconds can be well approximated by a $d$-dimensional LDS, then the total number of free variables to estimate is approximately $d(T/\tau)$ for the state sequence and  $Kd$ for the observation matrix. An SPC operating at $1/f_s$ samples per second obtains a total of $T f_s$ compressive measurements.
If CS-LDS were employed at a compression ratio of $C$, then 
\[ C T f_s \ge c_0 \frac{d T}{\tau} + c_1 K d \log (N /K) \implies f_s \ge c_0 \frac{d}{C \tau} + c_1 \frac{K d}{C T} \log(N/K). \]
The key dependence here are on how $d$, $K$ and $\tau$ change as a function of $N$.
In particular, even if $d$ and $K$ increased as $\sqrt{N}$, then $f_s$ would need be scale linearly in $N$ to maintain the same compression level.

{\flushleft {\bf Connection to affine-rank minimization: }}
The pioneering work of Fazel \cite{fazel2002matrix} in developing convex optimization techniques
to system identification problems has interesting parallels to the ideas proposed in this paper.
One of the key ideas espoused in \cite{fazel2002matrix} is that, when the video sequence $\bfy_{1:T}$ is an LDS, the block Hankel matrix $\rm{Hank} (\bfy_{1:T}, q)$ is low rank.
When we have linear measurements of the video frames, we can solve an affine-rank 
problem to recover the video.
However, such methods optimize on the Hankel matrix directly and lead to computationally infeasible designs even for videos of very small dimensions.
In contrast, CS-LDS has been shown to be fast and computationally feasible for very large videos involving millions of variables. 
The key is our two-step solution that isolates the space of unknowns into two manageable sets and solves for each separately. 

{\flushleft {\bf  Universality: }}
An attractive property of random matrix-based CS measurement is the universality of the measurement process. 
Universality implies that the sensing process is independent of the subsequent reconstruction algorithm. This makes the sensing design ``future-proof''; for such systems, if we devise a more sophisticated and powerful recovery algorithm in the future, then we do not need to redesign the camera or the sensing framework.
The CS-LDS framework violates this property. 
The two-step measurement process of Section \ref{sec:cslds},
which is key to breaking the bilinearity introduced by the LDS prior, implies that the CS-LDS design is not universal.
An intriguing direction for future research is the design of a universal CS-LDS measurement process.

{\flushleft {\bf Online tracking:}} We have made  the assumption of a static observation matrix $C$. However, as the length of the video increases, the assumption of a static $C$ is satisfied only by increasing the state space dimension.
An alternate approach is to allow for a time-varying observation matrix $C(t)$ and track it from the compressive measurements. This would give us the benefit of a low state space dimension and yet, be accurate when we sense for long durations. 

\begin{figure}[!ttt]
\includegraphics[width=\textwidth]{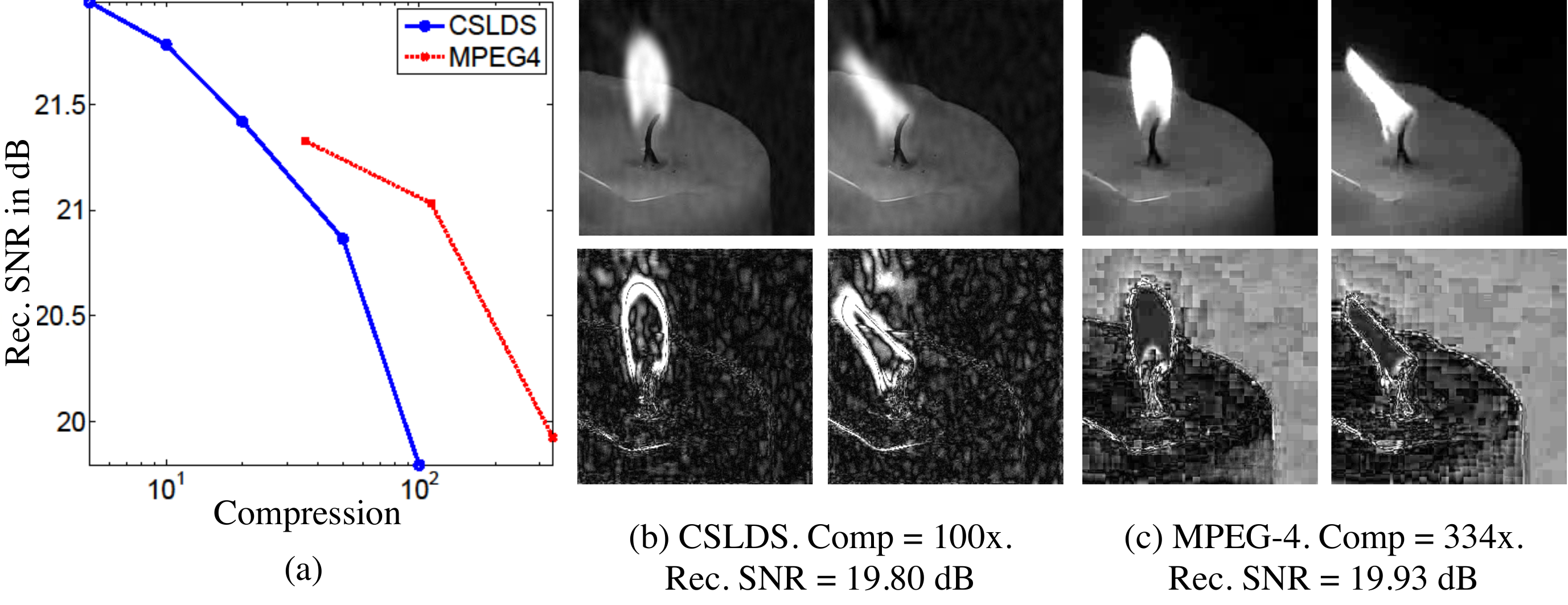}
\caption{Performance of MPEG-4 video compression and CS-LDS on the ``candle'' video  (see Figure \ref{fig:resolution}). Shown are (a) reconstruction SNR at various compression ratios, and (b, c - top) a few reconstructed frames  and (b,c - bottom) error in reconstruction magnified 10$\times$. It is worth nothing that the MPEG-4 algorithm has complete access to the ground truth video, while CS-LDS works purely with undersampled linear measurements of the video. None-the-less, even at the same reconstruction SNR, the quality of MPEG-4 recovery is significantly better. This can be attributed to the non-linear and adapted coding that seeks to mitigate errors that are perceptually dominant.}
\label{fig:mpeg}
\end{figure}

{\flushleft{\bf Beyond LDS:}}
Figure \ref{fig:mpeg} captures the relative performance of MPEG-4 compression algorithm and CS-LDS on a video. MPEG-4 has access to the ground truth video and, as a consequence, it achieves significantly better compressions for the same performance in recovery (see Figure \ref{fig:mpeg}(a)).
Further, it is worth noting that the non-linear encoding in MPEG-4 produces errors that are {\em imperceptible} and hence, even at the same level of reconstruction error, produces videos that are of higher visual quality (see Figure \ref{fig:mpeg}(b,c)).
This points at the inherent drawbacks of a linear encoder.
While the CS-LDS framework makes a compelling case study of LDSs for video CS, its applicability to arbitrary videos is limited.
In particular, it does not extend to simple non-stationary scenes such as people walking or panning cameras (see the result associated with Figure \ref{fig:amalgam2}(h)). This motivates the search for models more general than LDS. In this regard, a promising line of future research is to leverage models from the video compression literature for  CS recovery.

\section*{Acknowledgments}
ACS and RGB were partially supported by the grants NSF CCF-0431150, CCF-0728867, CCF-0926127, CCF-1117939, ARO MURI W911NF-09-1-0383, W911NF-07-1-0185, DARPA N66001-11-1-4090, N66001-11-C-4092, N66001-08-1-2065, ONR N00014-12-1-0124 and AFOSR FA9550-09-1-0432.

RC was partially supported by the Office of Naval Research under the Grant N00014-12-1-0124.

\bibliographystyle{siam}
\bibliography{videocs}

\begin{thebibliography}{10}

\bibitem{cslds}
{\em {CS-LDS} {P}roject webpage}.
\newblock URL = {http://www.ece.rice.edu/\~{}as48/research/cslds}.

\bibitem{ayazoglu2011dynamic}
{\sc M.~Ayazoglu, B.~Li, C.~Dicle, M.~Sznaier, and O.~I. Camps}, {\em Dynamic
  subspace-based coordinated multicamera tracking}, in {IEEE} Intl. Conf. Comp.
  Vision, 2011.

\bibitem{baraniuk2008model}
{\sc R.~G. Baraniuk, V.~Cevher, M.~F. Duarte, and C.~Hegde}, {\em Model-based
  compressive sensing}, {IEEE} Trans. Inf. Theory, 56 (2010), pp.~1982--2001.

\bibitem{baraniuk2008simple}
{\sc R.~G. Baraniuk, M.~Davenport, R.~DeVore, and M.~Wakin}, {\em A simple
  proof of the restricted isometry property for random matrices}, Constr.
  Approx., 28 (2008), pp.~253--263.

\bibitem{brockett1970finite}
{\sc R.~W. Brockett}, {\em Finite {D}imensional {L}inear {S}ystems}, Wiley,
  1970.

\bibitem{candes2009exact}
{\sc E.~J. Cand{\`e}s and B.~Recht}, {\em Exact matrix completion via convex
  optimization}, Found. Comp. Math., 9 (2009), pp.~717--772.

\bibitem{candes2006robust}
{\sc E.~J. Cand{\`e}s, J.~Romberg, and T.~Tao}, {\em Robust uncertainty
  principles: Exact signal reconstruction from highly incomplete frequency
  information}, {IEEE} Trans. Inf. Theory, 52 (2006), pp.~489--509.

\bibitem{candes2010power}
{\sc E.~J. Cand{\`e}s and T.~Tao}, {\em The power of convex relaxation:
  Near-optimal matrix completion}, {IEEE} Trans. Inf. Theory, 56 (2010),
  pp.~2053--2080.

\bibitem{cevher2008compressive}
{\sc V.~Cevher, A.~C. Sankaranarayanan, M.~F. Duarte, D.~Reddy, R.~G. Baraniuk,
  and R.~Chellappa}, {\em Compressive sensing for background subtraction}, in
  Euro. Conf. Comp. Vision, Oct. 2008.

\bibitem{Chan2005}
{\sc A.~B. Chan and N.~Vasconcelos}, {\em Probabilistic kernels for the
  classification of auto-regressive visual processes}, in {IEEE} Conf. Comp.
  Vision and Pattern Recog, June 2005.

\bibitem{chikuse2003statistics}
{\sc Y.~Chikuse}, {\em Statistics on special manifolds}, Springer Verlag, 2003.

\bibitem{coifman2001noiselets}
{\sc R.~Coifman, F.~Geshwind, and Y.~Meyer}, {\em Noiselets}, Appl. Comp. Harm.
  Anal., 10 (2001), pp.~27--44.

\bibitem{ding2007rank}
{\sc T.~Ding, M.~Sznaier, and O.~I. Camps}, {\em A rank minimization approach
  to video inpainting}, in {IEEE} Intl. Conf. Comp. Vision, 2007.

\bibitem{donoho2006compressed}
{\sc D.~L. Donoho}, {\em Compressed sensing}, {IEEE} Trans. Inf. Theory, 52
  (2006), pp.~1289--1306.

\bibitem{doretto2003dynamic}
{\sc G.~Doretto, A.~Chiuso, Y.~N. Wu, and S.~Soatto}, {\em {Dynamic textures}},
  Intl. J. Comp. Vision, 51 (2003), pp.~91--109.

\bibitem{duarte2008single}
{\sc M.~F. Duarte, M.~A. Davenport, D.~Takhar, J.~N. Laska, T.~Sun, K.~F.
  Kelly, and R.~G. Baraniuk}, {\em {Single-pixel imaging via compressive
  sampling}}, {IEEE} Signal Process. Mag., 25 (2008), pp.~83--91.

\bibitem{duarte2013measurement}
{\sc M.~F. Duarte, M.~B. Wakin, D.~Baron, S.~Sarvotham, and R.~G. Baraniuk},
  {\em Measurement bounds for sparse signal ensembles via graphical models},
  {IEEE} Trans. Inf. Theory, 59 (2013), pp.~4280--4289.

\bibitem{fazel2002matrix}
{\sc M.~Fazel}, {\em Matrix rank minimization with applications}, PhD thesis,
  Stanford University, 2002.

\bibitem{fazel2001rank}
{\sc M.~Fazel, H.~Hindi, and S.~P. Boyd}, {\em A rank minimization heuristic
  with application to minimum order system approximation}, in {IEEE} Amer.
  Control Conf., June 2001.

\bibitem{fazel2003log}
\leavevmode\vrule height 2pt depth -1.6pt width 23pt, {\em Log-det heuristic
  for matrix rank minimization with applications to hankel and euclidean
  distance matrices}, in {IEEE} Amer. Control Conf., June 2003.

\bibitem{grant2011cvx}
{\sc M.~Grant and S.~Boyd}, {\em {CVX: M}atlab software for disciplined convex
  programming, version 1.21}, Available at {http://cvxr. com/cvx},  (2011).

\bibitem{haupt2006signal}
{\sc J.~Haupt and R.~Nowak}, {\em {Signal reconstruction from noisy random
  projections}}, {IEEE} Trans. Inf. Theory, 52 (2006), pp.~4036--4048.

\bibitem{hitomi2011video}
{\sc Y.~Hitomi, J.~Gu, M.~Gupta, T.~Mitsunaga, and S.~K. Nayar}, {\em {V}ideo
  from a single coded exposure photograph using a learned over-complete
  dictionary}, in {IEEE} Intl. Conf. Comp. Vision, Nov. 2011.

\bibitem{kreutz2003dictionary}
{\sc K.~Kreutz-Delgado, J.~F. Murray, B.~D. Rao, K.~Engan, T.~W. Lee, and T.~J.
  Sejnowski}, {\em Dictionary learning algorithms for sparse representation},
  Neural Comp., 15 (2003), pp.~349--396.

\bibitem{nayar2006programmable}
{\sc S.~K. Nayar, V.~Branzoi, and T.~E. Boult}, {\em Programmable imaging:
  Towards a flexible camera}, Intl. J. Comp. Vision, 70 (2006), pp.~7--22.

\bibitem{needell2009cosamp}
{\sc D.~Needell and J.~A. Tropp}, {\em Cosamp: Iterative signal recovery from
  incomplete and inaccurate samples}, Appl. Comp. Harm. Anal., 26 (2009),
  pp.~301--321.

\bibitem{park2009multiscale}
{\sc J.~Y. Park and M.~B. Wakin}, {\em A multiscale framework for compressive
  sensing of video}, in Pict. Coding Symp., May 2009.

\bibitem{pati1993orthogonal}
{\sc Y.~C. Pati, R.~Rezaiifar, and P.~S. Krishnaprasad}, {\em {Orthogonal
  matching pursuit: Recursive function approximation with applications to
  wavelet decomposition}}, in Asilomar Conf. Signals Sys. Comp., Nov. 1993.

\bibitem{dynTex}
{\sc R.~P{\'e}teri, S.~Fazekas, and M.J. Huiskes}, {\em {DynTex:} {A}
  comprehensive database of dynamic textures}, Pattern Recog. Letters, 31
  (2010), pp.~1627--1632.

\bibitem{recht2007guaranteed}
{\sc B.~Recht, M.~Fazel, and P.~A. Parrilo}, {\em Guaranteed minimum-rank
  solutions of linear matrix equations via nuclear norm minimization},
  arXiv:0706.4138,  (2007).

\bibitem{reddy2011p2c2}
{\sc D.~Reddy, A.~Veeraraghavan, and R.~Chellappa}, {\em {P2C2:} {P}rogrammable
  pixel compressive camera for high speed imaging}, in {IEEE} Conf. Comp.
  Vision and Pattern Recog, June 2011.

\bibitem{saad1981krylov}
{\sc Y.~Saad}, {\em Krylov subspace methods for solving large unsymmetric
  linear systems}, Math. Comput., 37 (1981), pp.~105--126.

\bibitem{saisanDWS01}
{\sc P.~Saisan, G.~Doretto, Y.~Wu, and S.~Soatto}, {\em Dynamic texture
  recognition}, in {IEEE} Conf. Comp. Vision and Pattern Recog, Dec. 2001.

\bibitem{sankaranarayanan2010compressive}
{\sc A.~C. Sankaranarayanan, P.~Turaga, R.~Baraniuk, and R.~Chellappa}, {\em
  Compressive acquisition of dynamic scenes}, in Euro. Conf. Comp. Vision, Sep.
  2010.

\bibitem{soatto2001dynamic}
{\sc S.~Soatto, G.~Doretto, and Y.~N. Wu}, {\em {Dynamic textures}}, in {IEEE}
  Intl. Conf. Comp. Vision, July 2001.

\bibitem{sznaier2012compressive}
{\sc M.~Sznaier}, {\em Compressive information extraction: A dynamical systems
  approach}, in System Identification, vol.~16, 2012, pp.~1559--1568.

\bibitem{turaga2009unsupervised}
{\sc P.~Turaga, A.~Veeraraghavan, and R.~Chellappa}, {\em Unsupervised view and
  rate invariant clustering of video sequences}, Comp. Vision and Image
  Understd., 113 (2009), pp.~353--371.

\bibitem{berg2008probing}
{\sc E.~van~den Berg and M.~P. Friedlander}, {\em Probing the pareto frontier
  for basis pursuit solutions}, {SIAM} J. Scientific Comp., 31 (2008),
  pp.~890--912.

\bibitem{van1994n4sid}
{\sc P.~Van~Overschee and B.~De~Moor}, {\em {N4SID:} {S}ubspace algorithms for
  the identification of combined deterministic-stochastic systems}, Automatica,
  30 (1994), pp.~75--93.

\bibitem{vaswani2008kalman}
{\sc N.~Vaswani}, {\em Kalman filtered compressed sensing}, in {IEEE} Conf.
  Image Process., Oct. 2008.

\bibitem{vaswani2009modified}
{\sc N.~Vaswani and W.~Lu}, {\em Modified-{CS}: Modifying compressive sensing
  for problems with partially known support}, in Intl. Symp. Inf. Theory, June
  2009.

\bibitem{veeraraghavan2006function}
{\sc A.~Veeraraghavan, R.~Chellappa, and A.~K. Roy-Chowdhury}, {\em The
  function space of an activity}, in {IEEE} Conf. Comp. Vision and Pattern
  Recog, June 2006.

\bibitem{veeraraghavan2011coded}
{\sc A.~Veeraraghavan, D.~Reddy, and R.~Raskar}, {\em Coded strobing
  photography: Compressive sensing of high speed periodic events}, {IEEE}
  Trans. Pattern Anal. Mach. Intell., 33 (2011), pp.~671--686.

\bibitem{Veeraraghavan2005}
{\sc A.~Veeraraghavan, A.~K. Roy-Chowdhury, and R.~Chellappa}, {\em Matching
  shape sequences in video with applications in human movement analysis},
  {IEEE} Trans. Pattern Anal. Mach. Intell., 27 (2005), pp.~1896--1909.

\bibitem{wakin2006compressive}
{\sc M.~B. Wakin, J.~N. Laska, M.~F. Duarte, D.~Baron, S.~Sarvotham, D.~Takhar,
  K.~F. Kelly, and R.~G. Baraniuk}, {\em Compressive imaging for video
  representation and coding}, in Pict. Coding Symp., Apr. 2006.

\bibitem{wakin2010observability}
{\sc M.~B. Wakin, B.~M. Sanandaji, and T.~L. Vincent}, {\em On the
  observability of linear systems from random, compressive measurements}, in
  {IEEE} Conf. on Decision and Control, Dec. 2010.

\end{thebibliography}

%%\begin{biography}[{\includegraphics[width=1in,height=1.25in,clip,keepaspectratio]{mshell.eps}}]{Michael Shell}

%\begin{biography}{Michael Shell}
%\end{biography}
\end{document}